\documentclass[letterpaper]{article} 
\usepackage{aaai25}  
\usepackage{times}  
\usepackage{helvet}  
\usepackage{courier}  
\usepackage[hyphens]{url}  
\usepackage{graphicx} 
\usepackage{natbib}  
\usepackage{caption} 
\usepackage{algorithm}

\usepackage{algpseudocode}
\usepackage{xcolor}
\usepackage{subcaption}
\usepackage{booktabs}
\usepackage{multirow}

\usepackage{newfloat}
\usepackage{listings}
\DeclareCaptionStyle{ruled}{labelfont=normalfont,labelsep=colon,strut=off} 
\floatstyle{ruled}
\newfloat{listing}{tb}{lst}{}
\floatname{listing}{Listing}
\title{Langevin Monte Carlo Beyond Lipschitz Gradient Continuity}
\author {
    Matej Benko\textsuperscript{\rm 1},
    Iwona Chlebicka\textsuperscript{\rm 2},
    J\o{}rgen Endal\textsuperscript{\rm 3},
    Błażej Miasojedow\textsuperscript{\rm 2}
}
\affiliations {
    \textsuperscript{\rm 1}Institute of Mathematics, Faculty of Mechanical Engineering, Brno University of Technology\\ Technická 2896/2, 616 69 Brno, Czech Republic\\
    \textsuperscript{\rm 2} Institute of Applied Mathematics and Mechanics, University of Warsaw \\ ul. Banacha 2, 02-097 Warsaw, Poland\\
    \textsuperscript{\rm 3}Department of Mathematical Sciences, Norwegian University of Science and Technology (NTNU)\\
N-7491 Trondheim, Norway\\
    matej.benko@vutbr.cz, i.chlebicka@mimuw.edu.pl, jorgen.endal@ntnu.no,
    b.miasojedow@mimuw.edu.pl
}
\usepackage{amsmath,amsfonts,amssymb,amsthm,
 mathrsfs,dsfont}
\usepackage{enumerate}

\DeclareMathOperator*{\argmin}{arg\,min}

\newtheorem{coro}{\bf Corollary}[section]
\newtheorem{theo}[coro]{Theorem}
\newtheorem{lem}[coro]{Lemma}
\newtheorem{prop}[coro]{Proposition}

\newtheorem{rem}[coro]{Remark}

\newcommand{\Xki}{X_{k+1}}

\newcommand{\proxv}{\mathrm{prox}_V^{\tau}}
\newcommand{\proxvk}{\mathrm{prox}_V^{\tau}}
         
\newcommand{\vk}{\varkappa}
\newcommand{\cF}{{\mathcal{F}}}
\newcommand{\cFV}{{\cF_V}}
\newcommand{\cFE}{{\cF_{\mathcal{E}}}}
\def\R{{\mathbb{R}}}
\def\Rd{{\R^d}}

\def\dmu{{\,\mathrm{d}\mu}}
\newcommand{\minim}{\mu^*}

\newcommand{\cP}{{\mathscr{P}}}
\def\dx{{\mathrm{d}x}}
\def\dy{{\mathrm{d}y}}
\def\d{{\mathrm{d}}}
\newcommand{\prox}{{\rm prox}}  
\def\N{{\mathbb{N}}}

\newcommand{\nve}{n_\ve}
\newcommand{\vr}{\varrho}

\newcommand{\ve}{\varepsilon}
\newcommand{\taue}{\tau_\ve}
\newcommand{\vrk}{\vr_k}
\newcommand{\id}{\mathsf{Id}}
\newcommand{\vrki}{\vr_{k+1}}
\newcommand{\Ex}{\mathbb{E}}

\newcommand{\ca}{{\mathsf{c}}}
\newcommand{\Ca}{{\mathcal{C}}}
\newcommand{\hvrki}{\hat{\vr}_{k+1}}
\newcommand{\hvrk}{\hat{\vr}_{k}}

\newcommand{\vrkt}{\vr_{k+\frac{1}{3}}}
\newcommand{\vrkd}{\vr_{k+\frac{2}{3}}}

\newcommand{\vrN}{\vr_{N}}
\newcommand{\vrNn}{\vr_{N+n}}

\newcommand{\kl}{\operatorname{KL}}
\newcommand{\nunN}{{\nu_n^N}}

\begin{document}

\maketitle

\begin{abstract}
{We present a significant advancement in the field of Langevin Monte Carlo (LMC) methods by introducing the Inexact Proximal Langevin Algorithm (IPLA). This novel algorithm broadens the scope of problems that LMC can effectively address while maintaining controlled computational costs. IPLA extends LMC's applicability to potentials that are convex, strongly convex in the tails, and exhibit polynomial growth, beyond the conventional $L$-smoothness assumption. 
Moreover, we extend LMC's applicability to super-quadratic potentials and offer improved convergence rates over existing algorithms. Additionally, we provide bounds on all moments of the Markov chain generated by IPLA, enhancing its analytical robustness.}
\end{abstract}

%
 \begin{links}
     \link{Code}{https://github.com/192459/lmc-beyond-lipschitz-gradient-continuity}
 \end{links}

\section{Introduction}

Langevin Monte Carlo methods are powerful tools for sampling and optimization in complex, high-dimensional problems across various fields, including machine learning, statistics, physics, and beyond.   When the exact form of a distribution is unknown or difficult to compute, efficient sampling is essential for high-dimensional data, where traditional sampling methods often become computationally prohibitive. The ability of such algorithms to scale efficiently with dimensionality and manage non-smooth elements enhances their applicability in various machine learning tasks, including regression, classification, and clustering, providing a powerful tool for developing robust and accurate models. Given the potential $V\colon \Rd\to \R$ and associated density 
\[
\minim(x)=\tfrac{\exp(-V(x))}{\int_{\Rd}\exp(-V(y))\,\dy}
\]
our goal is to generate samples $X_1,\dots,X_n$ approximately from $\minim$. The ability to perform effective sampling from a~distribution with a density known up to a constant is crucial in Bayesian statistical inference \citep{Gelman_book,Robert_Casella}, machine learning \citep{Andrieu2003}, ill-posed inverse
problems \citep{Stuart}, and computational physics \citep{Krauth}.

One of the most efficient solutions to the problem of generating a sample from $\minim$ is Langevin Monte Carlo (LMC), which was designed via discretization of Overdamped Langevin Equation
\begin{equation}\label{eq:langevin}
\d Y_t=-\nabla V(Y_t)\d t +\sqrt{2}\ \d B_t\,
\end{equation}
where $B_t$ is a $d$-dimensional Brownian motion. Classically LMC is derived by Euler--Maruyama discretization of the above diffusion process. Namely, we iteratively define Markov Chains such that for a given $X_0$ and all $k\geq 0$ we have
\(
X_{k+1}=X_k-\tau_k\nabla V(X_k)+\sqrt{2\tau_k}\,\overline {Z}_{k+1}\,,
\)
where $\{\tau_k\}_{k\geq 0}$ is sequence of stepsizes and $\{\overline{Z}_{k+1}\}_{k\geq0}$ are i.i.d standard $d$-dimensional Gaussian random variables which are independent of $X_0$. In the recent years the LMC algorithm has been intensively investigated see \citep{DT,chewi22a,dalalyan17a,Dalalyan17,DMM,DKR,DM17,DM19,erdogdu21a,erdogdu22a,mousavi-hosseini23a,VW}. Efficiently scaling with dimension ensures that machine learning models remain computationally feasible and accurate as the complexity of the data increases. To the best of our knowledge, the overall complexity of generating one sample from $\mu^*$ with given precision $\ve$ scales linearly with dimension, see \cite{DMM}. More precisely, it is proven therein that in the case of a convex potential $V$ an error in the Kullback--Leibler divergence scales linearly with dimension. The same dependence on dimension is shown therein for error expressed in the Wasserstein distance when $V$ is strongly convex. 

The abovementioned  studies and related research are typically conducted under the assumption that $V$ is $L$-smooth, i.e., $\nabla V$ is $L$-Lipschitz. A particular consequence of such an assumption is that a quadratic function majorizes $V$. This narrows the scope of possible applications excluding many classical potentials behaving like polynomials with higher power than $2$, such as the one generating the processes of Ginzburg--Landau, cf. \cite{BDMS,goldenfeld}. Henceforth, aiming at capturing the natural potentials we relax the $L$-smoothness condition trading it with $\lambda$-convexity outside a ball. 

We stress that it is necessary to assume that $\nabla V$ is $L$-Lipschitz  while  considering Euler--Maruyama scheme, because otherwise the generated chain is transient \citep{RT96,MSH}.
Therefore, to tackle non-$L$-smooth potentials one is forced  to modify the scheme. The natural idea how to adjust the scheme comes from a variational formulation of \eqref{eq:langevin} by \citet{JKO,AGS}:
\begin{equation}\label{eq:minimizer}
\minim =\argmin_{\mu\in\cP_2(\Rd)}\cF[\mu]\,,
\end{equation}
where the functional $\cF$ is defined for
$\mu\in\cP(\R^d)$ (probabilistic measures) via the following formula 
\begin{equation*}
    \cF[\mu]:=\cFV[\mu]+  \cFE[\mu]\,,
\end{equation*}
where (denoting $\dmu(x)=\mu(x)\,\dx$ when a density exists)
\begin{flalign*}
\cFV[\mu]&:=\int_{\R^d}V(x)\dmu(x)\,,\\
\cFE[\mu]&:=\begin{cases}
        \int_{\R^d}\mu(x)\log\mu(x)\,\dx\,, &\text{if }\mu\in\cP_{\rm ac}\,,\\
        +\infty\, ,&\text{otherwise}\,.
    \end{cases}
\end{flalign*}
Note that $\minim $ is a minimizer of $\cF$. Indeed, due to \cite[Lemma~1]{DMM} we know that if $\minim \in \cP_2(\Rd)$, then $\cFV [\minim] < +\infty$ and $\cFE[\minim]  < + \infty$. Moreover, for $\mu \in \cP_2(\Rd)$ satisfying $\cFE[\mu] < + \infty$, it holds
\(\cF[\mu] - \cF[\minim] = \kl (\mu | \minim) \,,
\)
where $\kl$ stands for the Kullback--Leibler divergence between measures (it is also called the relative entropy). Since $\kl (\mu | \minim)\geq 0$ and the equality holds only if $\mu=\minim$  we get that $\minim$ is a minimizer of $\cF$. By this observation the LMC algorithm can be seen as an inexact optimization of $\cF$, see also \citep{bernton18a,DMM}. 

The typical approach in order to find a minimizer of the variational functional $\cF$ is going along the gradient flow in the Wasserstein space.  We split the gradient flow of $\cF$ into two gradient flows (of $\cFE$ and of $\cFV$). By \citet[Theorem 11.1.4]{AGS} and \citet[Theorem 2.1.5]{oksendal2010stochastic} the gradient flow of $\cFE$ is the standard Brownian motion.  To describe the gradient flow of $\cFV$ we introduce the {\it proximal operator} defined for the function $V$ and the time step $\tau$ by 
\begin{equation*}
    \proxv (x) := \argmin_{y\in\R^d} \left\{ V(y) + \tfrac{1}{2\tau} \lvert y-x \rvert^2\right\} \, .
\end{equation*} 
By \citet[Proposition 10.4.2(ii)]{AGS} the gradient flow of $\cFV$ starting from $\mu$ with time step $\tau$ is given by $(\proxv )_\#\mu$. Therefore, the splitting algorithm for minimizing $\cF$ is the Proximal Langevin Monte Carlo (PLMC) given by 
\[
X_{k+1}=\proxv(X_k)+\sqrt{2\tau_k}\,\overline{Z}_{k+1}\,.
\]
 The definition of $\proxv(x)$ 
 implies $V(\proxv(x))\leq V(x)$, which  prevents PLMC from exploding.
 PLMC was already studied for machine learning problems by \citet{bernton18a,SKL} and in the context of numerical solutions to aggregation-diffusion equations by \citet{BCEM}. The main drawback of PLMC is that it is typically not feasible to perform the proximal step exactly, so this step should be approximated. However, $\cFV$ and $\cFE$ are contractions in the Wasserstein space and an error introduced by an inexact proximal step does not accumulate significantly. In addition, the function to be optimized in the proximal step is strongly convex and there are several accurate and fast computational methods to approximate $\proxv$. 
 
 In the current paper, we propose and give theoretical guarantees that the Inexact Proximal Langevin Algorithm (IPLA for short) can be used successfully to generate from $\minim$ a sample beyond the assumption of global Lipschitz gradient continuity of $V$. The typical assumption of the strong convexity of $V$ is also relaxed.

\section{Our Contribution}

We make significant progress in the field of Langevin Monte Carlo methods by introducing a novel algorithm handling essentially broader scope of problems, than treated up to now, and preserving controlled computational cost.  Namely, we propose the Inexact Proximal Langevin Algorithm IPLA (Algorithm~\ref{alg:upl}). We extend the current scope of LMC algorithms by effective handling a large and natural class of potentials beyond the typical assumption of $L$-smoothness. 

We assume that the potential $V$ is convex, strongly convex in tails, and with polynomial growth with power $q_V+1\geq 2$. Under such hypothesis, when an error of the inexact proximal step is of order $\tau^2$, then IPLA generates one sample from $\minim$ with accuracy $\ve$ in the $\kl$-divergence with total complexity $d^{\frac{q_V+1}{2}}\mathcal{O}(\ve^{-2})$, see Theorem~\ref{thm:kl_bound} and Corollary~\ref{coro:kl}. Moreover, when $V$ is strongly convex, we can allow an error of the inexact proximal step of order $\tau^{\frac{3}{2}}$. In this setting we obtain that the total number of iterations to get one sample with the error $\ve$  in the Wasserstein distance is of order  $d^{\frac{q_V+1}{2}}\mathcal{O}(\ve^{-2})$, see Theorem~\ref{theo-main_conv} and Corollary~\ref{coro:was}.
Let us note that a lower accuracy of approximation of the proximal step in the strongly convex regime is expected. Since the proximal step is a gradient flow of $\cFE$, which is contracting in the Wasserstein space, the introduced error accumulates linearly. However, for strongly convex $V$ the gradient flow of $\cFE$ is strongly contracting, and the error is progressively suppressed during the iterations.

Our results provide better convergence rates than bounds under similar assumptions for the Tamed Unadjusted Langevin Algorithm, cf. \citet{BDMS}. Moreover, we get the same complexity as the Unadjusted Barker Algorithm~\citep{livingstone2024}, but under weaker assumptions -- we do not assume one-sided Lipschitz conditions on $\nabla V$. We note that when $q_V=1$ we retrieve dependence on the dimension that is the best-known for LMC, see \cite{DMM}. 

Let us recall that the known results for LMC cover only potentials with growth between linear and quadratic. Our contribution is complementary and embrace super-quadratic potentials. Henceforth, we complete the scope of an application of LMC to all (sufficiently smooth) convex potentials of power growth.  Besides the convergence rates of IPLA, we show a bound on all moments of the Markov chain $\{X_k\}_{k\geq0}$ generated by IPLA, see Theorem~\ref{theo:moments}. Such bounds are not only crucial for analyzing IPLA, but are also of separate interest for analysis of Langevin Monte Carlo algorithms~\cite{CGM,DM19,livingstone2024}.

\section{Description of our setting}
We consider an optimization problem on the space of probability measures over $\Rd$ for which we use the notation $\cP(\Rd)$. We define $\cP_{\rm{ac}}(\Rd)$ -- the subspace of $\cP(\Rd)$ consisting of measures absolutely continuous with respect to the Lebesgue measure. By $\cP_m(\Rd)$, for $m\geq 1$, we denote the subspace of $\cP(\Rd)$ consisting of measures with finite $m$-th moment, i.e.,\begin{flalign*}
    \cP_m(\Rd)&:=\{\mu\in\cP(\Rd)\colon \int_\Rd |x|^m\mu(\dx)<\infty\}\\
 &=\{\mu\in\cP(\Rd)\colon \mu(|\cdot|^{m})<\infty\}
\,.
\end{flalign*}
We denote by $C^1(\Rd,\R)$ the space of functions from $\Rd$ to $\R$ with continuous derivatives.

A continuously differentiable function $f:\R^d \rightarrow  \R$ is called  \begin{itemize}
    \item $\lambda$-{\it convex} ({\it strong convex})  {\it for $\lambda\in\R$  on $\R^d$}, if for all $x,y \in \R^d$ it satisfies
\begin{equation*}
f(x)\geq f(y)+\nabla f(y) \cdot (x-y)+\tfrac{\lambda}{2} |x-y|^2\,; 
\end{equation*}
\item $\lambda$-{\it convex for $\lambda\in\R$  outside of a ball $B\subset\R^d$}, if for all $x,y \in \R^d$ it satisfies
\[
f(x)\geq f(y)+\nabla f(y) \cdot (x-y)+\tfrac{\lambda}{2}\mathds{1}_{\R^d\setminus B}(y)|x-y|^2\,.
\]
\end{itemize}

If $f$ is $\lambda$-convex, the function $f-\tfrac{\lambda}{2}|\cdot|^2$ is $0$-convex, so we immediately get that
$$
\big(\nabla f(x)-\nabla f(y)\big)\cdot (x-y)\geq \lambda|x-y|^2\,.
$$ 

We denote by $\overline Z_k$ the standard Gaussian distribution on $\Rd$. For the sake of clarity of the exposition, throughout the paper we denote by $Z_k$ the rescaled one, namely $Z_k \sim \mathcal{N}(0, 2\tau \, \id)$, where $\tau$ will be clear from the context.\newline

 We will analyze IPLA in terms of two distances between measures. The first of them is the Kullback-Leibler divergence, also known as relative entropy, given for any measure $\mu,\nu\in \cP(\Rd)$ by
\[
\kl(\mu|\nu):= \begin{cases}
    \int_{\Rd}\frac{d\mu}{d\nu}(x)\log\left(\frac{d\mu}{d\nu}(x)\right)\nu(\dx)\,,& \text{if } \mu\ll \nu\,,\\
    +\infty\,, & \text{otherwise}\,.
\end{cases}
\]
Let us note that due to the Jensen inequality for all measures, $\mu,\nu\in\cP(\Rd)$, $\kl(\mu|\nu)\geq 0$ and equality holds only if $\mu=\nu$. The second distance used in our analysis is the  2-Wasserstein distance defined for all $\mu,\nu\in \cP_2(\Rd)$ by
\begin{equation}\label{def:was}
W^2_2(\mu,\nu):=\inf_{\gamma\in\Pi(\mu,\nu)}\int_\Rd |x-y|^2\gamma(\dx,\dy)\,,
\end{equation}
where $\Pi(\mu,\nu)$ denotes the coupling between $\mu$ and $\nu$, that is, $\Pi(\mu,\nu)$ is a set of probability measures $\gamma\in\cP_2(\Rd\times\Rd)$ such that for every measurable set $A$ it holds $\gamma(A,\Rd)=\mu(A)$ and $\gamma(\Rd,A)=\nu(A)$. We observe that by \citet[Theorem 4.1]{Villani} there exists an optimal coupling $\gamma^*$, for which the infimum in \eqref{def:was} is obtained. 

By the pushforward of the function $f:\Rd\to\Rd$, we mean the measure $f_\#$ defined as $f_\#\nu (U):=\nu(f^{-1}(U))$ for any measure $\nu\in\cP(\Rd)$ and any measurable set $U\subseteq\Rd$.\newline

\goodbreak

\noindent Let us present our regime.
\begin{description}
\item[{(V)}] For the potential $V$ we assume what follows:\begin{itemize}
    \item $V\in C^1(\R^d,\R)$ and is convex on $\R^d$.
    \item $V$ is $\lambda_V$-convex for $\lambda_V>0$ outside a given ball $B_V(0,R_V)$ with the center in $0$ and a radius $R_V\geq 0$. 
    \item  The minimum of $V$ is attained at $x^*=0$.
    \item  For $q_V\geq 1$, $C_V>0$, and for all $x,y \in \Rd$ it holds
\begin{flalign}\nonumber
V(y)\leq & V(x)+\nabla V(x)\cdot(y-x)\\+&C_V\left(1+|x|^{q_V-1}+|y|^{q_V-1}
\right)\vert y-x \vert^2\, . \label{ass-V-main}
\end{flalign}\end{itemize}  

\item[{($\boldsymbol{\vr_0}$)}] The initial distribution $\vr_0\in\cP_{m_0}(\Rd)$ has a~finite   moment of order $m_0=q_V+1$.
\end{description} 

Let us comment on the assumptions {\bf (V)} and {\bf ($\boldsymbol{\vr_0}$)}:
\begin{enumerate}[{\it (i)}]
    \item Assumption  {\bf (V)} embraces convex potentials with tails of polynomial growth with power $q_V+1\geq 2$.  The common assumption in the literature for LMC allows for convex functions with growth bounded by square function. 
    \item The assumption that $x^*=0$ is imposed just to simplify the presentation. All results remain true when $x^*\neq 0$ is any vector. There is no need for $x^*$ to be the unique minimizer of $V$.
    \item In assumption  {\bf (V)} condition \eqref{ass-V-main} is satisfied for $q_V=1$ if $\nabla V$ is $L$-Lipschitz with $C_V=\tfrac L2$. Taking $q_V>1$ allows for treating way broader class of potentials $V$.    
    \item If $\nabla V$ is locally Lipschitz with $q_V$-power growth, i.e., if there exists $L_q>0$ such that for all $x,y \in \R^d$ it holds
\begin{flalign*}
|\nabla V&(x)-\nabla V(y)|\\&\leq L_q \min\{|x-y|,1\}(1+|x|^{q_V-1}+|y|^{q_V-1})\,,
\end{flalign*} then \eqref{ass-V-main} is satisfied for $C_V=2^{2q_V-2} L_q $, cf. \cite[Lemma~A.2]{BCEM}.
\item In assumption {\bf ($\boldsymbol{\vr_0}$)} we require finite $(q_V+1)$-th moment of the initial distribution. This is a very weak assumption, as typically the finiteness of higher moments is used to prove the existence and uniqueness of solutions to Langevin equation~\eqref{eq:langevin}, see \citet{CGM}. Moreover, such an assumption is not restrictive in the application, as the initial measure is chosen arbitrarily in practice.
\end{enumerate}

\section{Inexact Proximal Langevin Algorithm}
Let us present the idea of our algorithm. By \eqref{eq:minimizer} we know that $\minim$ is a minimizer of $\cF$ over a space of measures, so we design an algorithm optimizing this functional. The basic idea to reach the minimizer is the analog of the gradient descent algorithm. In our case the functional $\cF$ lives on the Wasserstein space, so -- instead of the classical gradient -- we need to employ the gradient flow of the functional. For definition and basic properties of the gradient flows on Wasserstein spaces we refer to \citet{AGS,Santambrogio-overview}. We note that the gradient flow of $\cF$ is not given by explicit formula, but we know that $\cF=\cFV+\cFE$ where the gradient flows of both $\cFV$ and $\cFE$ are well-understood. As mentioned in the introduction, the gradient flow of $\cFV$ is given by a pushforward of the proximal operator, i.e., $(\proxvk)_\#$.  The gradient flow of $\cFE$ is the standard Brownian motion. Therefore, we use the splitting algorithm of the form
\[X_{k+1}=\prox^{\tau_k}_V(X_k)+\sqrt{2\tau_k}\,\overline {Z}_{k+1}\,,\] which can be called Proximal Langevin Algorithm. However, in practice, computing the exact proximal operator is challenging. Therefore, we perform its more feasible numerical approximation. Since under our assumptions, both flows are contracting the introduced error will not accumulate too much and can be controlled. 
\begin{algorithm}[ht!]
\caption{Inexact Proximal Langevin Algorithm (IPLA)}\label{alg:upl}
\textbf{Initialize: }Sample initial distribution $X_0 \sim \varrho_0$
\begin{algorithmic}
\For{$k=0, \ldots , n-1$}

\smallskip

\State \textit{Step 1:} Run routine for computing
 $\proxvk(X_k)
$
\[\text{with an output }X_{k+\frac{2}{3}}=\proxvk(X_k)+\Theta_{k+\frac{2}{3}}\,.\]

\State \textit{Step 2:} Add Gaussian noise $g_\tau$, i.e
\[X_{k+1}=X_{k+\frac{2}{3}}+Z_{k+1},\ \ Z_{k+1}\sim\mathcal{N}(0, 2\tau  \id)\,.\]
\EndFor
\end{algorithmic}
\end{algorithm}

The IPLA (Algorithm~\ref{alg:upl}) is defined as follows: in each step we perform two half steps. Given the previous state $X_k$, first we move according to the inexact proximal step of a function $V$ with stepsize $\tau$, where the error $\Theta_{k+\frac{2}{3}}$ (which could be stochastic or deterministic) is bounded by $\delta$. Next we add an independent of history and centered Gaussian random variable with covariance $2\tau\id$. 

For the theoretical analysis, we will split the inexact proximal step into two substeps: exact proximal step and additive error. Further, we will use the following notation.
\begin{enumerate}[{\it (i)}]
\item $X_{k+\frac{1}{3}} := \prox_V^\tau (X_k)$,\
\item $X_{k+\frac{2}{3}} := X_{k+\frac{1}{3}} + \Theta_{k+\frac{2}{3}}; \ \Theta_{k+\frac{2}{3}} \sim \xi_\delta$,
\item $X_{k+1} := X_{k+\frac{2}{3}} + Z_{k+1}; \ Z_{k+1}\sim g_\tau$.
\end{enumerate}
In view of densities $\vr_k$ s.t. $X_k \sim \vr_k$ it reads as follows
\begin{enumerate}[{\it (i)}]
\item $\vr_{k+\frac{1}{3}} := (\prox_V^\tau)_{\#} \vr_k$,
\item $\vr_{k+\frac{2}{3}} := \vr_{k+\frac{1}{3}}*\xi_\delta  $,
\item $\vr_{k+1} := \vr_{k+\frac{2}{3}} * g_\tau$,
\end{enumerate}
where $\xi_\delta$ is a~probability measure supported on $B(0,\delta)$ and $g_\tau$ denotes the density of Gaussian distribution given by
\(
Z_k \sim \mathcal{N}(0, 2\tau \, \id) \, .
\)
We will also assume that for some $\kappa>0$ and $\alpha\geq 0$ the error bound $\delta$ and stepsize $\tau$ satisfy $\delta= \kappa\tau^{1+\alpha}$.

\section{Theoretical results} 
In this section, we provide theoretical guarantees for the accuracy of IPLA. We start with bounds for moments of the~Markov chain $\{X_k\}_k$ generated by IPLA, see Theorem~\ref{theo:moments}. Then we show error bounds of IPLA in $\kl$-divergence and in Wasserstein distance, see Theorems~\ref{thm:kl_bound} and~\ref{theo-main_conv}, respectively.

Now we show that provided that for $m\geq 0$ the $m$-th moment of initial measure $\vr_0$ is bounded, it holds $\sup_k\Ex |X_k|^m<\infty$. Such bounds are required since we allow super quadratic growth of $V$.
\begin{theo}[Moment bound]\label{theo:moments} 
   Let $m\geq 0$. 
Suppose that  $V$ satisfies~ {\bf (V)} and $\vr_0$ satisfies ($\boldsymbol{\vr_0}$), and, for $k=1,\dots,n$, let $X_k$ be as in Algorithm~\ref{alg:upl}. Then there exists a constant $\Ca_{m}>0$ such that for all  $0<\tau<1/\lambda_V$, we have  
   \[
   {\rm sup}_k \Ex |X_k|^m\leq  \Ca_{m}\min\{1,\lambda_V^{-m}\}{d^{\frac{m}{2}}}\,.
   \]
   If the error and the stepsize satisfy $\delta= \kappa\tau^{1+\alpha}$ for some $\kappa>0$ and $\alpha\geq 0$, then $\Ca_{m}=\Ca_{m}(m,\vr_0,\kappa,\lambda_V)$. Moreover,   $\Ca_{m}\ll +\infty$ when $\lambda_V \to 0$.
\end{theo} 
The proof of the above theorem and the exact value of $\Ca_m$ are given in Appendix~B. Let us observe that since $V$ has super-quadratic growth in tails, $\minim$ is sub-Gaussian distribution with all moments finite.
\begin{rem} 
\rm
If the minimum of $V$ is not at $x^*=0$, then the result would have been
$$
{\rm sup}_k \Ex |X_k-x^*|^m\leq  \Ca_{m}\min\{1,\lambda_V^{-m}\}{d^{\frac{m}{2}}}\,.
$$
\end{rem}
In order to prove the main result on the error bound of IPLA in $\kl$-divergence,  we mimic the classical reasoning from the convex optimization based on the following observation. Let us consider $\{x_k\}_{k\geq0}$ as an output of the first-order algorithm for minimizing convex function $f$ over $\Rd$. It is possible to obtain an inequality of the following form  
\[
2\tau(f(x_{k+1})-f(x^*))\leq |x_{k}-x^*|^2-|x_{k+1}-x^*|^2+ C\tau^2\,
\]
where $x^*$ is minimizer of $f$ and $C>0$ is some constant, see for example \citet{Beck}. Such an inequality is a key step in the proofs of convergence of certain algorithms. We provide the  counterpart of such a result in the form relevant to analyze IPLA on the Wasserstein space. It employs the quantity
\begin{multline} 
    K(\tau) := 2^{4q_V-3} L_{q_V} \big( (1 + \Ca_{q_V - 1} d^{\frac{q_V-1}{2}} \lambda_V^{-\frac{q_V - 1}{2}} )  2 d \tau\\ \quad+ \Ca_{q_V+1} d^{\frac{q_V+1}{2}} \tau^{\frac{q_V+1}{2}} \big) \, .\label{K:tau}
\end{multline}
\begin{rem}\label{rem:ktau}
There exists a constant $C_{q_V}<\infty$ such that for $\tau \leq 1 $ it holds
\(
K(\tau)  \leq C_{q_V}  \tau d^{\frac{q_V+1}{2}}\, .\)
\end{rem}
The mentioned auxiliary result reads as follows.
\begin{prop}\label{prop:main_inequality} Assume that the function $V: \Rd \to \R$, $d\geq 1$, satisfies assumption  {\bf (V)} with some $q_V\geq 1$, the initial measure $\vr_0$ satisfies {\bf ($\boldsymbol{\vr_0}$)}, and $\nu\in\cP_{q_V+1}(\Rd)$ is arbitrary.
Then there exist $C(\nu)>0$, such that
\begin{flalign*}
2\tau(\cF[\vrki]-\cF[\nu])\leq& W_2^2(\vrk,\nu)-W_2^2(\vrki,\nu)\\
&+ C(\nu)\delta+ K(\tau)\tau\,,
\end{flalign*}
where $C(\nu)$ is given in Appendix~B, while  $K$   in~\eqref{K:tau}.
\end{prop}
The proof is postponed to Appendix~B.\newline

We state our main results for  a~sequence of probability measures  $\{ \nu_n^N \}_{n \in \N}$, called average measures, defined for every $n,N \in \N, \ n \geq 1$ by
\(\nunN := \frac{1}{n}\sum_{k=N+1}^{N+n}  \vrk\,,\)
where $N$ is a burn-in time.
By the use of Proposition~\ref{prop:main_inequality} we infer the following main theorem on convergence of IPLA.
\begin{theo}[$\kl$-error bound] \label{thm:kl_bound}
Suppose that  $V$ satisfies~{\bf (V)} and $\vr_0$ satisfies ($\boldsymbol{\vr_0}$). Let $\tau < 1/\lambda_V$, $\kappa,\alpha>0$, and $\delta\leq \kappa\tau^{1+\alpha}$. Then it holds:
\begin{multline*}
\kl (\nunN | \minim) \leq  \tfrac{1}{2 n\tau} W_2^2 (\vrN, \minim) - \tfrac{1}{2n\tau}  W_2^2 (\vrNn, \minim)\\  
+ C(\minim)\kappa\tau^\alpha+ K(\tau)   \,,
\end{multline*}
where $K$ is defined by \eqref{K:tau} and $C(\minim)$ is the same as in Proposition~\ref{prop:main_inequality} evaluated in $\nu=\minim$.
\end{theo}

\begin{proof}
We use the convexity of $\kl$-divergence \citep[Theorem 2.7.2.]{cover} and get
\begin{equation*}
\kl (\nunN | \minim) \leq \frac{1}{n} \sum_{k=N+1}^{N+n}  \kl (\vrk | \minim) \, .
\end{equation*}
Applying Proposition~\ref{prop:main_inequality} 
we are able to write  
\begin{multline*}
\kl (\nunN | \minim) \leq \frac{1}{2n\tau} \sum_{k=N+1}^{N+n}\left({W_2^2(\vrk,\nu)-W_2^2(\vrki,\nu)}\right)\\+ C(\minim)\kappa\tau^\alpha+ K(\tau)  \,. 
\end{multline*}
\end{proof}
A direct consequence of Theorem~\ref{thm:kl_bound} is the following result on the overall complexity of generating one sample from $\minim$ with  accuracy $\ve$.
\begin{coro}\label{coro:kl}
Suppose that the assumptions of Theorem~\ref{thm:kl_bound} are satisfied.
Assume further that 
$$
0<\taue\leq\min\left\{\Big(\frac{\ve}{3C(\minim)\kappa}\Big)^{\frac{1}{\alpha}},1\right\}
$$ and
$K(\taue) \leq {\ve}/{3}\,$. Let the number of iterations $\nve$ be such that $\nve \geq { 3W_2^2 (\vr_0, \minim)}/{(2\ve \taue)} \,$. Then 
\begin{equation}\label{eq:kl_eps}\kl (\nu_{n_\varepsilon}^0 | \minim) \leq \ve \, .\end{equation}

Moreover,  for computing one sample  in terms of $\kl$ with precision $\ve$,  in a case of warm start ($W_2^2 (\vr_0, \minim)\leq C$ for some absolute constant $C$), we need  \begin{enumerate}[{\it (i)}]
    \item  $d^{\frac{q_V+1}{2}}\mathcal{O}(\ve^{-2})$ iterations, if $\alpha\geq 1$;
    \item  $d^{\frac{q_V+1}{2}}\mathcal{O}(\ve^{-1-\alpha^{-1}})$ iterations, if $\alpha<1$.
\end{enumerate} 
\end{coro}
\begin{proof}
    The inequality \eqref{eq:kl_eps} is a straightforward consequence of Theorem~\ref{thm:kl_bound}. To get estimates of computational complexity it is enough to observe that by Remark~\ref{rem:ktau} we have that condition $K(\taue) \leq \ve$ is satisfied if 
\begin{equation*}
0<\taue \leq \min \Big \{ {\ve}{C_{q_V}^{-1}} d^{-\frac{q_V+1}{2}},1\Big \} \, . 
\end{equation*} 
\end{proof}

Finally, let us show that, when the potential $V$ is globally $\lambda_V$-convex, we can refine our result by improving the convergence rates of IPLA. Based on results from \citet{BCEM} we have the following estimate on precision of IPLA in terms of Wasserstein distance.
\begin{theo}[Wasserstein error bound]\label{theo-main_conv}
Suppose that the potential $V$  satisfies  \textbf{(V)} for $R_V=0$ and $\lambda_V>0$, and that the initial measure $\vr_0$ satisfies ($\boldsymbol{\vr_0}$). Let $\tau < 1/ \lambda_V$, $\alpha\geq 0$, and $\delta\leq \kappa\tau^{1+\alpha}$. Then for all $k\in\N$ it holds
\begin{multline*}
W_2^2 (\vrk, \minim) \leq  2\big( 1 - \tfrac{\tau \lambda_V}{2} \big)^k W_2^2 (\vr_0, \minim)   + \tfrac{4}{\lambda_V} K(\tau)\\ + 2\kappa^2\tau^{2+2\alpha}\left(  \frac{1-\mathrm{e}^{-\lambda_V\tau (k-1)}}{1 - \mathrm{e}^{-\lambda_V\tau}} \right)^2\,.
\end{multline*} 
\end{theo}
This leads to the following complexity bounds.
\begin{coro}\label{coro:was}
 Suppose that the assumptions of Theorem~\ref{theo-main_conv} are satisfied. Assume further that \[\taue^{2\alpha}\leq \frac{\lambda_V^2\ve}{96\kappa^2\log^2(6 W_2^2 (\vr_0, \minim){\ve^{-1}})}\,\]and $K(\taue) \leq \tfrac{1}{12}\lambda_V\ve\,$. Let the number of iterations $\nve$ be such that \begin{multline*}
    2\log(6 W_2^2 (\vr_0, \minim){\ve^{-1}})\taue^{-1}\lambda_V^{-1} \leq \nve \\\leq4\log(6 W_2^2 (\vr_0, \minim){\ve^{-1}})\taue^{-1}\lambda_V^{-1}  \,.
\end{multline*} 
Then 
\[W_2^2 (\vr_{n_\varepsilon}, \minim) \leq \ve \, . 
\]
 Moreover, for computing one sample  in terms of the Wasserstein distance   with precision $\ve$,  in the case of warm start, up to logarithmic terms we need 
 \begin{enumerate}[{\it (i)}]
    \item  $d^{\frac{q_V+1}{2}}\mathcal{O}(\ve^{-2})$ iterations, if $\alpha\geq \tfrac{1}{2}$;
    \item  $d^{\frac{q_V+1}{2}}\mathcal{O}(\ve^{-\alpha^{-1}})$ iterations, if $\alpha<\tfrac{1}{2}$.
\end{enumerate} 
\end{coro}
\begin{proof}
Let us note that 
\[
\big( 1 - \tfrac{\tau \lambda_V}{2} \big)^{\nve}\leq\exp(-\nve\tfrac{\tau \lambda_V}{2})\text{ and } \frac{1-\mathrm{e}^{-\lambda_V\tau (\nve-1)}}{1 - \mathrm{e}^{-\lambda_V\tau}} \leq \nve\,.
\]
With the above bounds, the result follows directly from Theorem~\ref{theo-main_conv}.
\end{proof}

Note that if in Corollary~\ref{coro:was}  the stepsize $\tau_\ve=\tau$ is fixed, then $W_2^2 (\vr_{n}, \minim) \leq \ve$ for all sufficiently large $n$.

\begin{rem} The assumption of global $\lambda_V$-convexity of $V$ decreases the computational cost of a single iteration of IPLA. We can approximate the proximal step with smaller precision and keep the complexity of the whole algorithm of the order $\ve^{-2}$. Indeed,  as a~consequence of $\lambda_V$-convexity of the functional $\cFE$ we get \begin{enumerate}[{\it (i)}]
    \item  in Theorem~\ref{thm:kl_bound} -- under assumption {\bf{(V)}} with $R_V>0$ --  to achieve complexity of order $\ve^{-2}$ we need $\delta\leq \kappa\tau^{2}$;
    \item  in Theorem~\ref{theo-main_conv} -- in the strongly convex case -- the order $\ve^{-2}$ we have for $\delta\leq \kappa\tau^{3/2}$.\end{enumerate} 
\end{rem}
\begin{rem}
The complexity of IPLA, as shown in Corollaries~\ref{coro:kl} and \ref{coro:was}, depends on the number of iterations requiring additional computations. Since the optimized function is strongly convex, the cost of one iteration with precision $\delta$ is $\mathcal{O}(d\log(\delta))$, as detailed in Appendix~E.\end{rem}
\section{Experiments}
In this section, we demonstrate the application of IPLA on $3$ examples implemented in Python. We analyze the convergence rates and bias of the algorithm compared to the two known related LMC algorithms, namely TULA (Tamed Unadjusted Langevin Algorithm by \citet{BDMS}) and ULA (Unadjusted Langevin Algorithm by \citet{DM19}). The codes of our simulations can be found at \texttt{https://github.com/192459/lmc-beyond\--lip\-schitz-\-gradient\--continuity}.\\ See Appendix~D for additional results.

We shall compare the relative error (RE) 
and the coefficient of variance (CV) of various methods. In this context, the RE of the estimator $\widehat\theta$ of the true value $\theta$ is $\left|\frac{\widehat\theta-\theta}{\theta}\right|$, while CV is $\frac{{\rm sd}(\widehat\theta)}{\theta}$, where ${\rm sd}$ stands for the standard deviation.

\subsection{Example 1: Distribution with Light Tails}
Let us start with a simple and natural case where the potential has non-Lipschitz gradient.
Our goal is to sample from the density 
\begin{equation*}
\minim (x) \propto \exp{\big(-\tfrac{|x|^4}{4} \big)} \, ,
\end{equation*}
which is a~stationary distribution of the process 
\begin{equation*}
\d Y_t = -Y_t^3 \, \d t + \sqrt{2} \, \d B_t \,.
\end{equation*}
We estimate the moments $\Ex |Y|^2$, $\Ex |Y|^4$, and $\Ex |Y|^6$ in dimension $d=10^3$. For error analysis as true values of the estimated quantities, we use the Metropolis--Hastings algorithm with $10^7$ iterations. Note that $\nabla V$ is not globally Lipschitz and $V$ is $1$-convex outside a ball, so the assumptions on the convergence of ULA are not satisfied, but the convergence of TULA and IPLA is ensured. 

 We have run $10^5$ samples with 
 burn-in time $10^4$. We considered two scenarios. An initial value in the first of them is $x_0 = 7 \times \boldsymbol{1}_d$ (start in tail), while in the second one considered  we start at $x_0=0$ (start in the minimizer of $V$). Each experiment has been repeated $100$ times.  In Table~\ref{table:light_tails} we compare the relative error and the coefficient of variance between the  algorithms. We can see that IPLA has the smallest RE and its~CV is comparable to ULA's.
\begin{table}[ht] 
\centering
\begin{tabular}{llllll} \toprule
\multirow{2}{*}{Mom.}& \multirow{2}{*}{Method} & \multicolumn{2}{c}{Start in tail} & \multicolumn{2}{c}{Start in $x_0=0$} \\
 &  & RE & CV & RE & CV \\ \midrule
\multirow{3}{*}{$\Ex |Y|^2$} & IPLA & $0.0027$ & $0.0019$ & $0.0006$ & $0.0018$\\
& TULA & $0.0047$ & $0.0016$ & $0.0030$ & $0.0019$\\
& ULA & \texttt{NaN} & \texttt{NaN} & $0.0020$ & $0.0018$ \\
 \midrule
\multirow{3}{*}{$\Ex |Y|^4$} & IPLA & $0.0054$ & $0.0039$ & $0.0025$ & $0.0036$\\
 & TULA & $0.0095$ & $0.0032$ & $0.0073$ & $0.0039$\\
 & ULA & \texttt{NaN}& \texttt{NaN} & $0.0028$ & $0.0035$ \\
 \midrule
 \multirow{3}{*}{$\Ex |Y|^6$}  & IPLA & $0.0081$ & $0.0058$ & $0.0047$ & $0.0054$\\
 & TULA & $0.0144$ & $0.0047$ & $0.0120$ & $0.0058$\\
 & ULA & \texttt{NaN} & \texttt{NaN} & $0.0032$ & $0.0053$\\
\bottomrule
\end{tabular}
\caption{Estimation of the moments of light tails distribution from Example 1.}
\label{table:light_tails}
\end{table}

In Figure~\ref{fig:trajectories} we compare the trajectories of algorithms for the selected coordinate ($i=1$).  We see that ULA blows up in a few iterations, which results from the fact that $\nabla V$ is not globally Lipschitz.  To suppress this phenomenon, TULA adjusts stepsizes according to normalized $\nabla V$, namely to  $\frac{\nabla V}{1+|\nabla V|}$. This, on the other hand, results in unnecessarily small stepsizes in the tails. This can be seen in Figure~\ref{fig:trajectories}. As shown on Figure~\ref{fig:time_step} IPLA is less sensitive to the choice of stepsizes $\tau$ than TULA.

\begin{figure}[ht!]
\centering
\begin{subfigure}{0.233\textwidth}
\includegraphics[width=\textwidth]{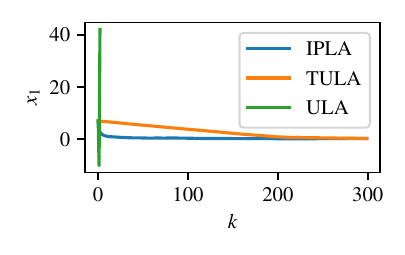}
\caption{}
\end{subfigure}
\begin{subfigure}{0.233\textwidth}
\includegraphics[width=\textwidth]{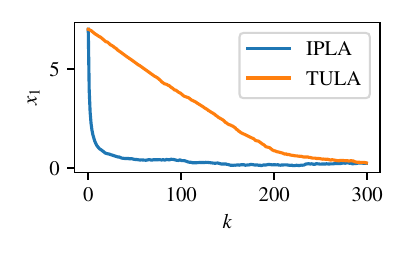}
\caption{}
\end{subfigure}
\caption{Trajectory of the $1^{\mathrm{st}}$ coordinate from Example 1 starting in a tail.  Both plots are based on the same data.}
\label{fig:trajectories}
\end{figure}

\begin{figure}[ht!]
\begin{subfigure}{0.233\textwidth}
    \includegraphics[width=\textwidth]{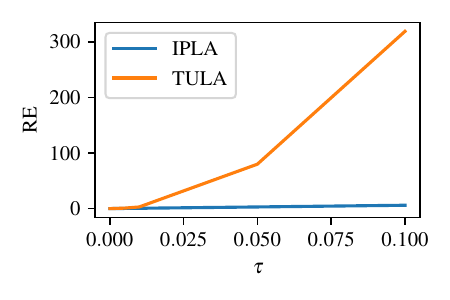}
    \caption{}
\end{subfigure}
\begin{subfigure}{0.233\textwidth}
    \includegraphics[width=\textwidth]{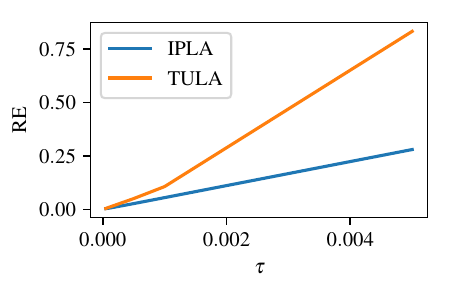}
    \caption{}
\end{subfigure} 
\caption{Dependence RE of IPLA and TULA (for Example~1) on a stepsize $\tau$ starting in a tail.}
\label{fig:time_step}
\end{figure}

\subsection{Example 2: Ginzburg--Landau Model}
We pass to a more complicated, but physically relevant example. 
Ginzburg--Landau model was introduced to describe phase transitions in physics, i.e., superconductivity, see \citet[Chapter 5]{goldenfeld}. 
It involves the potential 
\begin{equation*}
V(x) = \sum_{i,j,k = 1}^{q} \tfrac{1-\upsilon}{2} x_{ijk}^{2} + \tfrac{\upsilon \vk}{2} |\widetilde{\nabla} x_{ijk}|^2 + \tfrac{\upsilon \varsigma}{4} x_{ijk}^{4} \, ,
\end{equation*}
where $
    \widetilde{\nabla} x_{ijk} = \big( x_{i+jk} - x_{ijk}, x_{ij+k} - x_{ijk} ,  x_{ijk+} - x_{ijk}\big) \, ,$
with the notation such that 
$i_{\pm} := i \pm 1 \,{\rm mod}\, q$ and similarly for $j,k$. We consider $\vk= 0.1$, $\varsigma=0.5$, $\upsilon=2$ and $q=5$. The dimension in this example is $d = q^3 = 125$. Note that $\nabla V$ is not globally Lipschitz and $V$ is $\lambda_V$-convex for $\lambda_V>0$, \citep{BDMS}. We have run  $2 \times 10^4$ samples with burn-in time $10^4$.  We consider two scenarios as in Example~1. As starting in tail, we consider the case $x_0 = (100,0,\ldots, 0)$ and starting in the minimizer we take $x_0 = 0$. Each experiment has been repeated $100$ times. \begin{table}[ht!]
\centering
\begin{tabular}{llllll} \toprule
\multirow{2}{*}{Mom.}& \multirow{2}{*}{Method} & \multicolumn{2}{c}{Start in tail} & \multicolumn{2}{c}{Start in $x_0=0$} \\
 &  & RE & CV & RE & CV \\ \midrule
\multirow{3}{*}{$\Ex |Y|^2$} & IPLA & $0.0025$ & $0.0244$ & $0.0748$ & $0.0786$\\
& TULA & $0.0067$ & $0.0213$ & $0.0859$ & $0.0739$\\
& ULA & \texttt{NaN} & \texttt{NaN} & $0.0727$ & $0.0697$ \\
 \midrule
\multirow{3}{*}{$\Ex |Y|^4$} & IPLA & $0.0053$ & $0.0491$ & $0.1425$ & $0.1558$\\
 & TULA & $0.0134$ & $0.0425$ & $0.1635$ & $0.1473$\\
 & ULA & \texttt{NaN}& \texttt{NaN} & $0.1385$ & $0.1399$ \\
 \midrule
 \multirow{3}{*}{$\Ex |Y|^6$}  & IPLA & $0.0083$ & $0.0742$ & $0.2040$ & $0.2323$\\
 & TULA & $0.0199$ & $0.0638$ & $0.2337$ & $0.2208$\\
 & ULA & \texttt{NaN} & \texttt{NaN} & $0.1980$ & $0.2117$\\
\bottomrule
\end{tabular}
\caption{Estimation of moments of Ginzburg--Landau model from Example~2. }
\label{table:ginzburg_landau}
\end{table}  From Table~\ref{table:ginzburg_landau} we see that the results are analogous to those of Example~1. In the scenario starting in the minimizer, we have results similar to ULA's and better than TULA's. In the case of the start in tail, ULA blows up.

\begin{figure*}[ht]
\centering
\begin{subfigure}{0.23\textwidth}
\includegraphics[width=\textwidth]{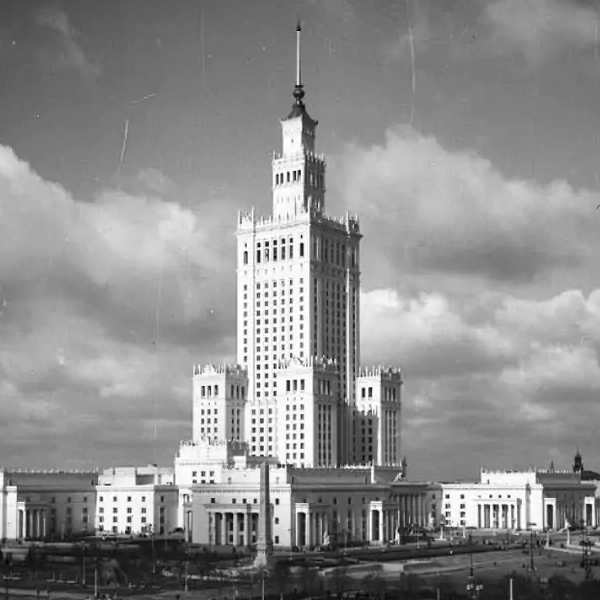}
\caption{Original image}
\end{subfigure}
\hspace{2mm}
\begin{subfigure}{0.23\textwidth}
\includegraphics[width=\textwidth]{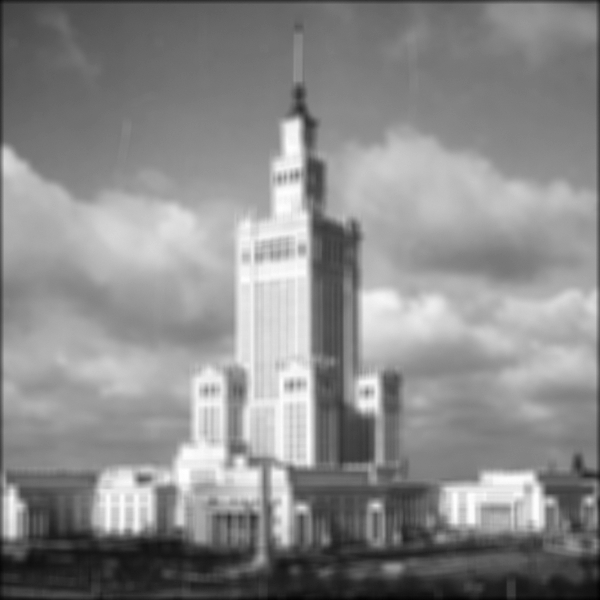}
\caption{Blurred with additive noise}
\end{subfigure}
\hspace{2mm}
\begin{subfigure}{0.23\textwidth}
\includegraphics[width=\textwidth]{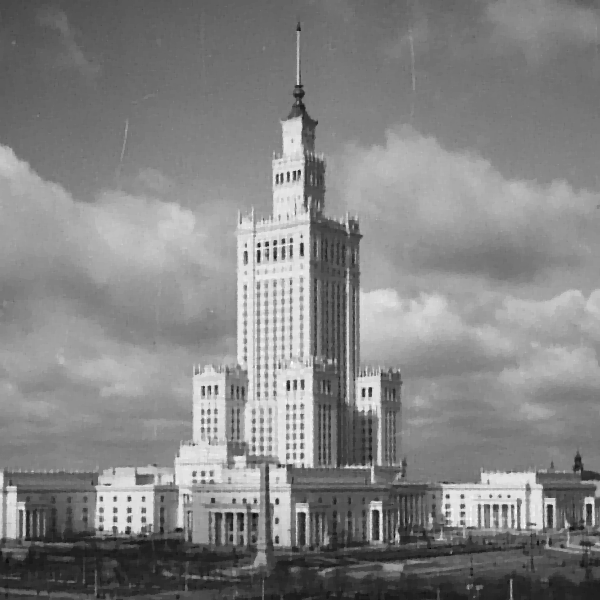}
\caption{Processed $(\beta = 0.03)$}
\end{subfigure}
\hspace{2mm}
\begin{subfigure}{0.23\textwidth}
\includegraphics[width=\textwidth]{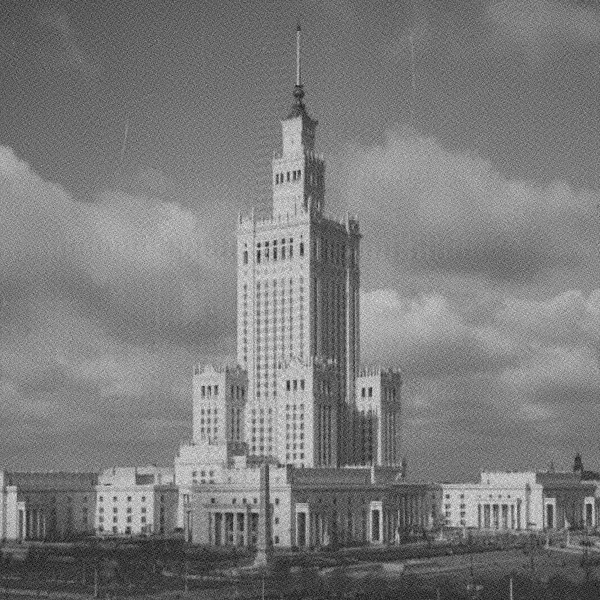}
\caption{Processed $(\beta = 0.0001)$}
\end{subfigure}
\caption{Result of the Bayesian Image Denoising from Example~3. The original photo by Zbyszko Siemaszko 1955-56. }
\label{fig:palace}
\end{figure*}
\subsection{Example 3: Bayesian Image Deconvolution}
We show now an example that IPLA is effective in complex high-dimensional scenarios. We consider {the Bayesian Image Deconvolution} problem inspired by Experiment 1 in \citet[Subsection 4.1.2]{durmus_bayesian_efficient}. 
Given a high-resolution grayscale image $x \in \Rd$ such that $d = n \times n=600 \times 600$ pixels from Figure~\ref{fig:palace}~(a),  the goal of the problem is to recover image $x$ from a blurred image. Namely, we are to sample from posterior density $p(x|y)$ of blurred image with additive noise (observation) such that 
\begin{equation*}
y = Hx + W \ \textup{ for }\ W \sim \mathcal{N} (0, \sigma^2 \id) \, ,
\end{equation*}
where $H$ is a circulant blur matrix representing 2D discrete convolution and $W$ is an additive noise. 

We need to define 2D discrete isotropic total variation via
\begin{flalign*}
TV(x) := &\sum_{i,j = 1}^{n-1} \sqrt{|x_{in, j+1}-x_{in, j}|^2+|x_{(i+1)n, j}-x_{in, j}|^2} \\
 + &\sum_{j = 1}^{n-1}|x_{nd, j+1}-x_{nd, j}| + \sum_{i = 1}^{n-1}|x_{(i+1)n, n}-x_{in, n}| \,,
\end{flalign*} which was proposed  in the context of image processing by \citet{ROC_model}. Note that $TV$ is a~discretization of the standard total variation norm, cf. \citet[Section 2]{chambonelle_tv}.

We choose prior distribution $p(x) \propto \exp (-\beta \, TV(x))$, where $\beta > 0$ is  a regularizing factor. Henceforth, we provide sampling from the posterior distribution $\minim:=p(x|y)$ satisfying 
\begin{equation*}
\minim  \propto \exp \Big(-\tfrac{1}{2\sigma^2}|y - Hx|^2 - \beta \, TV(x) \Big) \, .
\end{equation*}
Even if $TV$ is not smooth, and consequently the theoretical convergence rates are not known for IPLA, TULA, or ULA, we show the performance of our algorithm. It might be expected that the inexact proximal step could be non-trivial in this case.  Nonetheless, it can be effectively solved with the use of the Adaptive Primal--Dual Hybrid Algorithm by \citet[Algorithm 1]{goldstein}. 

Our results are presented on Figure~\ref{fig:palace}. On (b) we show blurred and noisy image $y$, where uniform circular blur matrix has depth $9$ and additive noise is of the standard deviation $\sigma=0.5$. On (c) and (d) we present posterior mean obtained by IPLA for $10^5$ iterations and with the desired precision of the inexact proximal step $\delta = 10^{-1}$.

\subsection{Effective implementation of IPLA}
 
 Despite involving the inexact proximal step,
 the computational complexity of IPLA is controlled. In Examples~1 and~2, we used the Newton conjugate gradient method with the provided analytically computed gradient and the Hessian matrix, making use of the sparsity of the Hessian matrix.  Using the standard Python library \texttt{SciPy} \citep{2020SciPy-NMeth},   IPLA has comparable performance to ULA. We write the time of computations on standard Macbook Air M1 2020. In Example~1 a~single run of $10^6$ iterations took from $5$ to $20 \, \mathrm{s}$  for ULA or TULA and from $15$ to $60 \, \mathrm{s}$, while in Example~2, all algorithms took from $3$ to $5 \, \mathrm{min}$ to run $10^4$ iterations.  In Example~3, the potential is not differentiable and we modify the existing implementation in the \texttt{PyProximal} library \citep{Ravasi2024}. Setting the initial value in the minimizer of $V$ reduces the cost of the inexact proximal step. Let us stress that in our simulation of Example~3, $10^5$ samples in the dimension $d=360\,000$ are obtained in less than $1\, \mathrm{min}$.

\subsection{Conclusion}
IPLA broadens the scope of problems that known LMC algorithms can address keeping low computational cost.
 
\bibliography{bib}
\section*{Acknowledgements}
I.C. is supported by NCN grant 2019/34/E/ST1/00120, B.M. is supported by NCN grant
2018/31/B/ST1/00253. The authors are grateful to Adam Chlebicki-Miasojedow and Sonja Letnes for close assistance in the progress of this research.
\newpage

\onecolumn
\section*{Appendix}
\renewcommand\thesection{\Alph{section}}
\setcounter{section}{0}
\setcounter{coro}{0}

\section{Auxiliary results}

The next lemma is crucial to deal with functions strongly convex only outside some ball.
\begin{lem}\label{lem:prox-contracts-1} Let $V$ be a function convex and differentiable on $\R^d$ and $\lambda_V$-convex outside a given ball $B_V$ for $\lambda_V\geq 0$ that attains its minimum at $x^*=0$. Then, for any $\tau>0$ and $x,z\in\R^d$, it holds
\begin{equation*}
2\tau\big(V(\proxv(x))-V(z)\big)\leq |x-z|^2-|x-\proxv(x)|^2-(1+\tau\lambda_V\mathds{1}_{\R^d\setminus B}(\proxv(x)))|\proxv(x)-z|\,.
\end{equation*}
Moreover,
for any $\tau>0$ and $x\in\R^d$ we have
    \[
   \left(1+\tau\lambda_V\mathds{1}_{\R^d\setminus B}(\proxv(x))\right) |\proxv(x)|^2\leq |x|^2\,.
    \]
\end{lem}

\begin{proof}
We denote $y=\proxv (x)$. Due to the $\lambda_V$-convexity of $V$ we know that 
\[
V(y)-V(z)\leq \nabla V(y) \cdot (y-z)-\tfrac{\lambda_V}{2}\mathds{1}_{\R^d\setminus B_V}(y)|y-z|^2\,.
\]
By the definition of $\proxv$ it follows that $\nabla V(y)=\tfrac{1}{\tau}(x-y)$. In turn
\[
2\tau\big(V(y)-V(z)\big)\leq 2(x-y)\cdot (y-z)-\tau\lambda_V\mathds{1}_{\R^d\setminus B_V}(y)|y-z|^2\,.
\]
To complete the result it suffices to note that $
|x-z|^2=|x-y|^2+|y-z|^2+2(x-y)\cdot(y-z)
$.
\end{proof}

To provide bounds on moments for a discrete-in-time Markov Chain defined by Algorithm~\ref{alg:upl} we will need the following auxiliary lemmas.
\begin{lem}\label{lem:moment_third}
Let $V$ be a function convex and differentiable on $\R^d$ and $\lambda_V$-convex outside a given ball $B_V := \{ x :\ |x| \leq R_V \}$ for $\lambda_V\geq 0$ and $m \geq 0$. Then for any $k\in\N$ and $\tau>0$ we have
\begin{equation*}
\Ex |X_{k+\frac{1}{3}}|^m \leq \big(\tfrac{1}{1+\tau \lambda_V} \big)^{\frac{m}{2}} \Ex |X_k|^m +  \tfrac{(1+\tau \lambda_V)^{\frac{m}{2}}-1}{(1+\tau \lambda_V)^{\frac{m}{2}}} \, R_V^m \, .
\end{equation*}
\end{lem}
\begin{proof} For $m=0$ there is nothing to prove. Suppose $m>0$. Since $\vrkt = (\proxvk)_{\#} \vrk$, we have
\begin{flalign*}
\Ex |X_{k+\frac{1}{3}}|^m &= \int_{\Rd} |x|^m \vrkt (\d x) = \int_{\Rd} |x|^m \, (\proxvk)_{\#} \vrk (\d x) 
=  \int_{A} |\proxvk (x)|^m \vrk (\d x) + \int_{\Rd\setminus A} |\proxvk (x)|^m \vrk (\d x)\,,
\end{flalign*}
where we denote $A:=\{x\colon \proxvk(x)\in B_V\}$. 
Inside the set $A$ we use a non-expansive property of proximal operator, namely $|\proxvk(x)| \leq |x|$, and the fact that $\vrk(A)\leq 1$. Consequently, we have that
\begin{flalign*}
    \int_{A} |\proxvk (x)|^m \vrk (\d x) &\leq \left(\frac{1}{1+\tau \lambda_V} \right)^{\frac{m}{2}}\int_{A} |\proxvk (x)|^m \vrk (\d x) + \frac{(1+\tau \lambda_V)^{\frac{m}{2}}-1}{(1+\tau \lambda_V)^{\frac{m}{2}}}\int_{A} |\proxvk (x)|^m \vrk (\d x) \\ &\leq \left(\frac{1}{1+\tau \lambda_V} \right)^{\frac{m}{2}}\int_{A} |x|^m \vrk (\d x) +\frac{(1+\tau \lambda_V)^{\frac{m}{2}}-1}{(1+\tau \lambda_V)^{\frac{m}{2}}} R_V^m \,.
\end{flalign*}
By Lemma~\ref{lem:prox-contracts-1} for $x \in \Rd \setminus A $ the following inequality holds 
\begin{equation*} 
|\proxvk(x)|^m \leq \big(\tfrac{1}{1+\tau\lambda_V}\big)^{\frac{m}{2}} |x|^m \, .
\end{equation*} Altogether we get 
\begin{equation*}
\Ex |X_{k+\frac{1}{3}}|^m \leq \frac{(1+\tau \lambda_V)^{\frac{m}{2}}-1}{(1+\tau \lambda_V)^{\frac{m}{2}}}R_V^m   + \left(\frac{1}{1+\tau \lambda_V} \right)^{\frac{m}{2}} \int_{\Rd } |x|^m \vrk (\dx) 
\end{equation*}
\end{proof}

\begin{lem}\label{lem:power-est}
Let $X,Z,\Theta\in\cP_m(\Rd),$ $m\geq 2$, be arbitrary and pairwise independent with $\Ex Z=0$. Then
\begin{equation}\label{eq:taylor_exp}
\Ex | X+ \Theta + Z |^m \leq \Ex | X |^m +  \ca (\Ex |X|^{m-1} \Ex|\Theta| +  \Ex |X|^{m-2} |Z+ \Theta|^2 + \Ex |Z+ \Theta|^m)\,.    
\end{equation}
\end{lem}
\begin{proof}
We compute the Taylor expansion of first order of the function $f(z) = |\xi + z|^m, m \geq 2$ around zero with reminder estimation. To do that, we compute the first and second derivative 
\begin{flalign*}
\nabla f (z) &= m |\xi + z|^{m-1} \tfrac{(\xi + z)}{|\xi + z|} =   m |\xi + z|^{m-2} (\xi + z) \, , \\
\nabla^2 f (z) &= 2 (\tfrac{m}{2} -1)m |\xi+z|^{m-4} (\xi+z)\otimes (\xi+z) + \id \, m |\xi+z|^{m-2} \,.
\end{flalign*}
We make use of the Taylor expansion with the Cauchy reminder such that 
\begin{equation*} 
f(z) \leq f(0) + \nabla f(0)\cdot  + |\tfrac{1}{2} z^\top \nabla^2 f(\vartheta z) z |  \quad\text{for some }\ \vartheta \in [0,1] \, .
\end{equation*}
We take $z=Z+\Theta$, $\xi=X$, and consider $\Ex [f(Z)]$. Then~\eqref{eq:taylor_exp} follows from $\Ex Z=0$ and elementary manipulation.
\end{proof}

\begin{lem} \label{lemma:ath_mom_bounds}
Let $Z \sim g_{\tau}$. Then for every $m\geq 0$, there exists $\widetilde{C}_m=\widetilde{C}_m(m)>0$, such that 
\begin{equation*}
\Ex |Z|^m \leq \widetilde{C}_m  \tau^{\frac{m}{2}} d^\frac{m}{2} \,. 
\end{equation*}
\end{lem}

\begin{proof}
Let $Z=(Z^1,\ldots, Z^d)$ and observe that for $m\geq 2$ it holds
$$
\Ex |Z|^m =\Ex\left|\sum_{i=1}^d |Z^i|^2\right|^\frac{m}{2}\leq d^{\frac{m}{2}-1}\sum_{i=1}^d \Ex|Z^i|^m=d^{\frac{m}{2}} \Ex|Z^1|^m\,.
$$
 We note that $\tfrac{1}{2\tau}(Z^1)^{2} \sim \chi^2(1) = \mathsf{Gamma}(\tfrac{1}{2}, \tfrac{1}{2})$,  $(Z^1)^2\sim \mathsf{Gamma}(\tfrac{1}{2}, \tfrac{1}{4\tau})$, and
\begin{align*}
\Ex |Z^1|^{m} &  = \int_{0}^{+\infty} y^{\frac{m}{2}} \frac{1}{(4 \tau)^\frac{1}{2} \Gamma (\frac{1}{2})} y ^{\frac{1}{2} -1 } \mathrm{e}^{-\frac{1}{2\tau} y} \d y = 
\frac{(4\tau)^\frac{1+m}{2} \Gamma \left(\frac{1+m}{2}\right)}{(4\tau)^\frac{1}{2} \Gamma (\frac{1}{2})} \int_{0}^{+\infty}\frac{1}{(4\tau)^\frac{1+m}{2} \Gamma \left(\frac{1+m}{2}\right)}y^{\frac{1+m}{2}-1} \mathrm{e}^{-\frac{1}{2\tau} y} \d y \, \\ &=4^{\frac{m}{2}}\tau^{\frac{m}{2}}\Gamma(\tfrac{1+m}{2})\Gamma(\tfrac{1}{2})^{-1}\,,
\end{align*} 
which completes the proof for $m\geq 2$. Otherwise, when $0<m<2$, by Jensen's inequality we obtain
\[\Ex|Z|^m\leq \big(\Ex |Z|^2\big)^\frac{m}{2}\leq  \big(d 4 \tau \Gamma(\tfrac{3}{2})\Gamma(\tfrac{1}{2})^{-1}  \big)^\frac{m}{2}\,.\]
\end{proof}

\begin{rem}\label{rem:ath_mom_bounds}\rm 
If $m=2$, from the proof of Lemma~\ref{lemma:ath_mom_bounds}, we infer that for $Z \sim g_{\tau}$ it holds $\Ex |Z|^2 =2\tau d$. 
\end{rem}

In the theoretical analysis, the key role is played by the following lemma from \cite{BCEM}.

\begin{lem}[{\citet{BCEM}}, Lemma 4.5] \label{lem:dmm5}
Let the function $V: \Rd \to \R$, $d\geq 1$, satisfy assumption  {\bf (V)} with $\lambda_V\geq 0$, $L_{q_V}>0$, and $q_V\geq 1$.  Assume further that  $\minim$ is a minimizer of $\cF$, $\tau>0$, $\vr_0$ satisfies {\bf($\boldsymbol{\vr_0}$)}, and for $k=0,\ldots,n-1$ we have $\vr_{k+\frac{1}{3}} = (\prox^\tau_{V})_{\#} \vr_k $, $ \vr_{k+\frac{2}{3}} = \vr_{k+\frac{1}{3}} * \xi_{\delta}$, and $ \vr_{k+1} = \vr_{k+\frac{2}{3}} * g_{ \tau}$.  Then: 
\begin{enumerate}[{(i)}]

    \item For any $\nu\in\cP_2(\R^d)\cap\cP_{q_V-1}(\R^d)$, we have \[
        2\tau\big(\cFV[\nu*g_{\tau}]-\cFV[\nu]\big)\leq 2^{4q_V-2}L_{q_V} \left(\left(1 + \nu(|\cdot|^{q_V-1})\right)2d\tau^2+\Ca_{q_V+1}d^{\frac{q_V+1}{2}}\tau^{\frac{q_V+3}{2}}\right)\,.\]
    \item For any $\nu\in \cP_2(\R^d)$ and $k\in\N$, we have \[2 \tau \big( \cFV [\vrkt] - \cFV [\nu]  \big) \leq W_2^2 (\vrk, \nu) - W_2^2 (\vrkt, \vrk)   
    \,.\]
    \item For any $ \nu \in \cP_2 (\R^d)$ and $k\in\N$, we have \[
    2\tau \big( \cFE [\vrki] - \cFE [\nu] \big) \leq W_2^2(\vrkd, \nu) - W_2^2 (\vrki, \nu)\,.\]
\end{enumerate}
\end{lem}

The next two lemmas bounds the error introduced by the inexact proximal step.
\begin{lem}\label{lem:lem4a}
   If $X\sim \mu$, $Y \sim \nu$, $\Theta \sim \rho$ and $|\Theta|\leq\delta$ a.s., then 
    \[W_2^2(\mu*\rho,\nu) \leq  W_2^2(\mu,\nu) + (2\Ex(|X|+|Y|)+\delta)\delta \,.\]
\end{lem}
\begin{proof}We choose the optimal coupling for $X,Y$ and  $\Theta$ independent on $X$ and $Y$. We then  observe that
    \begin{flalign*}
        W_2^2(\mu*\rho,\nu)&\leq \Ex|X+\Theta-Y|^2=\Ex|X-Y|^2+2\Ex\langle \Theta, X-Y\rangle +\Ex|\Theta|^2\leq\Ex|X-Y|^2+2\delta\Ex(|X|+|Y|)+\delta^2\\
        &=W_2^2(\mu,\nu)+ (2\Ex(|X|+|Y|)+\delta)\delta\,.
    \end{flalign*}
\end{proof}

\begin{lem}\label{lem:prox_inexact} Assume that the function $V: \Rd \to \R$, $d\geq 1$, satisfies assumption  {\bf (V)}. Let $\xi_\delta$ be a~probability measure supported on $B(0,\delta)$ and let $\nu\in\cP_{q_V}(\Rd)$. Then 
    \[
    |\cFV{[\nu]}-\cFV [\xi_\delta*\nu]|\leq C_V\left(\nu(|\cdot|)+\nu(|\cdot|^{q_V})\delta+\delta\nu(|\cdot|^{q_V-1})+\delta^{q_V}\right)\, \delta\,.
    \]
\end{lem}
\begin{proof}
    We note that by assumption {\bf(V)} it holds\begin{flalign*}
        \int_{\Rd}\int_{B(0,\delta)}V(y)-V(y-z)&\,\d\xi(z)\,\d\nu(y)\leq \int_{\Rd}\int_{B(0,\delta)}\left(\nabla V(y)\cdot z+C_V|z|^2\left(1+|y|^{q_V-1}+|z|^{q_V-1}\right)\right)\d\xi(z)\,\d\nu(y)\\
        &\leq \int_{\Rd}\int_{B(0,\delta)}\Big(\big|\nabla V(y)-\nabla V(0)\big||z| +C_V\left[|z|^2+|z|^2|y|^{q_V-1} +|z|^{q_V+1}\right]\Big)\,\d\xi(z)\d\nu(y)\\
        &\leq \int_{\Rd}\int_{B(0,\delta)}\Big(C_V|y|(1+|y|^{q_V-1})|z| +C_V\left[|z|^2+|z|^2|y|^{q_V-1} +|z|^{q_V+1}\right]\Big)\,\d\xi(z)\d\nu(y)\\
        &\leq C_V\left[\nu(|\cdot|)\delta+\nu(|\cdot|^{q_V})\delta+\delta^2+\delta^2\nu(|\cdot|^{q_V-1})+\delta^{q_V+1} \right]\\
        &\leq C_V\left(\nu(|\cdot|)+\nu(|\cdot|^{q_V})\delta+\delta\nu(|\cdot|^{q_V-1})+\delta^{q_V}\right)\, \delta\,.
    \end{flalign*}
 
\end{proof}

\section{Proofs of main results of the paper}\label{app:2:proofs}

\begin{proof}[Proof of Theorem~\ref{theo:moments}] We show a detailed proof in the case $\lambda_V\leq 1$. For $\lambda_V\geq 1$ it follows from the same arguments upon a few obvious changes. We concentrate now on justifying that for some $\Ca_m$, such that $\Ca_m\ll +\infty$ when $\lambda_V \to 0$, it holds
\begin{equation}\label{eq:step-3}
A_{k+1} := \Ex |\Xki|^m \leq  {\Ca_m}\lambda_V^{-m}d^{\frac{m}{2}}\,.
\end{equation} Steps~1 and~2 provide \eqref{eq:step-3} for $m\in[0,2]$, while in Step 3 we concentrate on $m>2$. Furthermore, in Step~3  we assume, as the induction condition, that~\eqref{eq:step-3} is satisfied for $m-1$ and we infer that the induction condition is true for $m$.\newline

\noindent 
\noindent {\bf Step 1. } We show that $\Ex |\Xki|^2 \leq \Ca_2\tfrac{d}{\lambda_V^2}$ for $\Ca_2$ such that $\Ca_2 \to \frac{36 \delta^2}{\tau^2}$ when $\lambda_V \to 0$.
We note that  
\begin{flalign*}
\Ex |X_{k+1}|^2 &= \Ex |X_{k+\frac{1}{3}} + \Theta_{k+\frac{2}{3}} + Z_{k+1}|^2 \\ &= \Ex |X_{k+\frac{1}{3}}|^2 + \Ex |\Theta_{k+\frac{2}{3}}|^2 + \Ex |Z_{k+1}|^2 + 2 \, \Ex [X_{k+\frac{1}{3}}\cdot \Theta_{k+\frac{2}{3}}] + 2 \, \Ex [X_{k+\frac{1}{3}} \cdot Z_{k+1}] + 2\, \Ex [\Theta_{k+\frac{2}{3}}\cdot Z_{k+1}] \\ &\leq \Ex |X_{k+\frac{1}{3}}|^2 + \Ex |\Theta_{k+\frac{2}{3}}|^2 + \Ex |Z_{k+1}|^2 + 2\delta \, \Ex |X_{k+\frac{1}{3}}|\, .
\end{flalign*}
Above we used that the random variables $X_{k+\frac{1}{3}}$, $\Theta_{k+\frac{2}{3}}$, $Z_{k+1}$ are pairwise independent, $\Ex [Z_{k+1}] = 0$, and that $|\Theta_{k+\frac{2}{3}} |\leq \delta$.
By Lemma~\ref{lem:moment_third} applied for $m=1$ and $m=2$, the fact that $\Ex |Z_{k+1}|^2 = 2\tau d$ and the estimate $\Ex |\Theta_{k+\frac{2}{3}}|^2 \leq \delta^2$, we have 
\begin{equation}\label{eq:EXk1_ineq}
\Ex |X_{k+1}|^2 \leq 
  \tfrac{1}{1+\tau \lambda_V}  \Ex |X_k|^2 +  \tfrac{\tau \lambda_V}{1+\tau \lambda_V} \, R_V^2+ \delta^2 + 2\tau d+
 2\delta \tfrac{1}{\sqrt{1+\tau \lambda_V}}  \Ex |X_k| +   2\delta\tfrac{\tau\lambda_V}{1+\tau\lambda_V+ \sqrt{1+\tau \lambda_V}} \, R_V  \,.
\end{equation}
We note that
\begin{equation*}
 \tfrac{1}{1+\tau \lambda_V} \Ex |X_k|^2  +
 2\delta \tfrac{1}{\sqrt{1+\tau \lambda_V}}  \sqrt{\Ex|X_k|^2} \leq \tfrac{2}{2+\tau \lambda_V}\, \Ex |X_k|^2 \quad \iff \quad 2 \delta \tfrac{(2 + \tau \lambda_V)\sqrt{1+\tau \lambda_V}}{ \tau \, \lambda_V} \leq \sqrt{\Ex |X_k|^2}\,.
\end{equation*}
In turn, by the Jensen inequality $\Ex |X_k| 
\leq \sqrt{\Ex |X_k|^2}$ and  \eqref{eq:EXk1_ineq} we get
\begin{flalign*}
\Ex |X_{k+1}|^2 \leq \max \Big\{ 4 \delta^2 \tfrac{(2 + \tau \lambda_V)^2 }{ \tau^2 \, \lambda_V^2 } + 4 \delta^2 \tfrac{2+\tau \lambda_V}{\tau \lambda_V} , \tfrac{2}{2+\tau \lambda_V}  \Ex |X_k|^2 \Big\}  + 
\tfrac{\tau \lambda_V}{1+\tau \lambda_V} \, R_V^2 +\tfrac{ 2\delta\tau\lambda_V}{1+\tau\lambda_V+\sqrt{1+\tau \lambda_V}} \, R_V + \delta^2+ 2\tau d  \, .
\end{flalign*}
To simplify the expression we consider the fact that $ 0 < \tau \lambda_V < 1 $ and get
\begin{equation*}
\Ex |X_{k+1}|^2 \leq \max \Big\{  \tfrac{36 \delta^2 }{ \tau^2 \, \lambda_V^2 } + \tfrac{12 \delta^2 }{\tau \lambda_V} , \tfrac{2}{2+\tau \lambda_V}  \Ex |X_k|^2 \Big\}  + 
\tau \lambda_V \, R_V^2 + \delta\tau\lambda_V \, R_V+ \delta^2 + 2\tau d  \, .
\end{equation*}
By iterating, since $\delta\leq \kappa\tau$, we get 
\begin{flalign*}
\Ex |X_{k+1}|^2 &\leq \max \Big\{  \tfrac{36 \delta^2 }{ \tau^2 \, \lambda_V^2 } + \tfrac{12 \delta^2 }{\tau \lambda_V} , \left(\tfrac{2}{2+\tau \lambda_V}\right)^{k+1}  \Ex |X_0|^2 \Big\}  + \left(
\tau \lambda_V \, R_V^2 + \delta\tau\lambda_V \, R_V + \delta^2+ 2\tau d \right)\sum_{j=1}^{k+1} \Big( \tfrac{2}{2+\tau \lambda_V} \Big)^j\leq
\tfrac{\Ca_2 }{\lambda_V^2}d \,,
\end{flalign*}
where $\Ca_2 := \tfrac{36\kappa^2}{d} + \tfrac{\lambda_V^2}{d}\Ex|X_0|^2 + \tfrac{16\kappa\delta+4 R_V^2\lambda_V + 4\delta R_V \lambda_V+ 8 d }{d} \lambda_V$. \newline

\noindent{\bf Step 2. } We aim to bound $\Ex|\Xki|^m$  for $m\in[0,2)$. We use the Jensen inequality and  Step~1 to obtain
\begin{equation*}
\Ex |\Xki|^m = (\Ex |\Xki|^2)^{\frac{m}{2}}\leq \left(\Ca_2 \tfrac{d}{\lambda_V^2}\right)^{\frac{m}{2}}\,.
\end{equation*}
\medskip

\noindent{\bf Step 3. } {Suppose $m>2$.}  Let us assume that \eqref{eq:step-3} holds for  all $m-1,m-2,\dots, m-l$, such that $m-l\geq 0$. By Lemma~\ref{lem:power-est}  with  $X:= X_{k+\frac{1}{3}}$ and $Z:=\Theta_{k+\frac{2}{3}} + Z_{k+1}$ we get 
\begin{flalign*}
A_{k+1} &:= \Ex |X_{k+\frac{1}{3}} + \Theta_{k+\frac{2}{3}} + Z_{k+1}|^m \\ &\leq \Ex|X_{k+\frac{1}{3}}|^m + \ca \Ex \Big[|X_{k+\frac{1}{3}} |^{m-1} |\Theta_{k+\frac{2}{3}}| \Big] +\ca \Ex \Big[|X_{k+\frac{1}{3}} |^{m-2} |\Theta_{k+\frac{2}{3}} +Z_{k+1}|^2 \Big] + \ca \Ex |\Theta_{k+\frac{2}{3}} +Z_{k+1}|^m \\ &=: \mathrm{I} + \mathrm{II} + \mathrm{III} + \mathrm{IV}\, .
\end{flalign*}
Note that the bound for $\mathrm{I}$ is covered by Lemma~\ref{lem:moment_third}.  As for term $\mathrm{II}$ we notice that \eqref{eq:step-3} holds for $m-1$ by induction.
Since $\Theta_{k+\frac{2}{3}}$ and $Z_{k+1}$ are assumed to be independent, $\Ex |\Theta_{k+\frac{2}{3}}| = \delta$, by Lemma~\ref{lem:moment_third} we have
\begin{flalign*}
\tfrac {1}{\ca}\mathrm{II}
&= \Ex\big[|X_{k+\frac{1}{3}} |^{m-1} \big] \, \Ex \big|\Theta_{k+\frac{2}{3}}\big|
 \leq \Big( \big(\tfrac{1}{1+\tau \lambda_V} \big)^{\frac{m-1}{2}} \Ex |X_k|^{m-1} +   \tfrac{(1+\tau \lambda_V)^{\frac{m-1}{2}}-1}{(1+\tau \lambda_V)^{\frac{m-1}{2}}} \, R_V^{m-1} \Big)\delta \\
&\leq \Big( \Ca_{m-1}\lambda_V^{-(m-1)}d^\frac{m-1}{2}  +    R_V^{m-1} \Big) \delta \leq d^\frac{m-1}{2}\Big( \Ca_{m-1}\lambda_V^{-(m-1)} +    R_V^{m-1} \Big)  \delta\,.
\end{flalign*}
As for $\mathrm{III}$, we notice the same as for $\mathrm{II}$ and then by Remark~\ref{rem:ath_mom_bounds} we get the inequality  
\begin{equation*}
\Ex |\Theta_{k+\frac{2}{3}} + Z_{k+1}|^2 \leq 2\Ex |\Theta_{k+\frac{2}{3}}|^2  + 2\Ex |Z_{k+1}|^2 \leq 2\delta^2 +   4\tau d \, .
\end{equation*}
Therefore, using Lemma~\ref{lem:moment_third} and since \eqref{eq:step-3} holds for $m-2$, we obtain 
\begin{flalign*}
\tfrac {1}{\ca} \mathrm{III} &\leq \Big(\big(\tfrac{1}{1+\tau \lambda_V} \big)^{\frac{m-2}{2}} \Ex |X_k|^{m-2} + \tfrac{(1+\tau \lambda_V)^{\frac{m-2}{2}}-1}{(1+\tau \lambda_V)^{\frac{m-2}{2}}} \, R_V^{m-2} \Big) \, (2\delta^2  + 4\tau d) \\
&\leq \Big(\Ca_{m-2} \lambda_V^{-(m-2)} d^{\frac{m-2}{2}} +  R_V^{m-2} \Big) \, (2\delta^2  + 4\tau d)\,.
\end{flalign*}
We find a bound for $\mathrm{IV}$ using Lemma~\ref{lemma:ath_mom_bounds}   such that 
\begin{flalign*}
\tfrac {1}{\ca}\mathrm{IV} &\leq 2^{m-1}\Ex | \Theta_{k+\frac{2}{3}} |^m +  2^{m-1}\Ex | Z_{k+1} |^m \leq 2^{m-1}\left(\delta^m +\widetilde{C}_m\tau^{\frac m2}d^{\frac m2}\right)\, .
\end{flalign*}
Altogether we have 
\begin{multline*}
A_{k+1} \leq \big(\tfrac{1}{1+\tau \lambda_V} \big)^{\frac{m}{2}} A_k +  \tfrac{(1+\tau \lambda_V)^{\frac{m}{2}}-1}{(1+\tau \lambda_V)^{\frac{m}{2}}} \, R_V^m\\ +\ca\left[d^\frac{m-1}{2}\Big( \Ca_{m-1}\lambda_V^{-m+1} +    R_V^{m-1} \Big) \delta +\Big(\Ca_{m-2} \lambda_V^{-m+2} d^{\frac{m-2}{2}} +  R_V^{m-2} \Big) \, (2\delta^2  + 4\tau d)
+ 2^{m-1}\left(\delta^m +\widetilde{C}_m\tau^{\frac m2}d^{\frac m2}\right)\right]\,.
\end{multline*}
By iteration and taking into account that $\big(\frac{1}{1+\tau \lambda_V}\big)^{\frac{m}{2}} \leq \frac{1}{1+\tau \lambda_V}$ for $m \geq 2$ we get 
\begin{multline*}
A_{k+1} \leq \big(\tfrac{1}{1+\tau \lambda_V} \big)^{(k+1)\frac{m}{2}} A_0 +  \tfrac{(1+\tau \lambda_V)^{\frac{m}{2}}-1}{(1+\tau \lambda_V)^{\frac{m}{2}}} \, R_V^m\sum_{j=0}^k\big(\tfrac{1}{1+\tau \lambda_V} \big)^{\frac{m}{2}j} \\ +\ca\left[d^\frac{m-1}{2}\Big( \Ca_{m-1}\lambda_V^{-m+1} +    R_V^{m-1} \Big) \delta +\Big(\Ca_{m-2} \lambda_V^{-m+2} d^{\frac{m-2}{2}} +  R_V^{m-2} \Big) \, (2\delta^2  + 4\tau d)
+ 2^{m-1}\left(\delta^m +\widetilde{C}_md^{\frac m2}\tau^{\frac m2}\right)\right]\times\\\times\sum_{j=0}^k\big(\tfrac{1}{1+\tau \lambda_V}\big)^j\,.
\end{multline*} 
We observe that $\sum_{j=0}^k\Big( \tfrac{1}{1+\tau \lambda_V} \Big)^{\frac m2 j}\leq\tfrac{(1+\tau \lambda_V)^{\frac{m}{2}}} {(1+\tau \lambda_V)^{\frac{m}{2}}-1}$ and $\sum_{j=0}^k\Big( \tfrac{1}{1+\tau \lambda_V} \Big)^{j}\leq \frac{2}{\tau\lambda_V}$ and recall that $\tau \lambda_V < 1$,  hence
\begin{flalign*}
  A_{k+1}& \leq A_0 +R_V^m +\ca\Big[2\, \kappa d^{\tfrac{m-1}{2}}\big( \Ca_{m-1}\lambda_V^{-m} + R_V^{m-1}\lambda_V^{-1} \big) \\ &\qquad +2(2\kappa\delta +4d)\big(\Ca_{m-2} \lambda_V^{-m+1} d^{\frac{m-2}{2}} +  R_V^{m-2}\lambda_V^{-1} \big) +2^m\Big(\kappa\delta^{m-1}\lambda_V^{-1} +\widetilde{C}_md^{\frac m2}
  \lambda_V^{-\frac m2}\Big) \Big]\,.
\end{flalign*}
We complete the arguments the same way as in Step 1.
\end{proof}
The direct consequence of Theorem~\ref{theo:moments} and Lemma~\ref{lem:dmm5} is the following fact.
\begin{coro} \label{coro:dmm5b}
If assumptions {\bf(V)} and {\bf($\boldsymbol{\vr_0}$)} are satisfied then for any  $\tau <\frac{1}{\lambda_V}$, it holds
\begin{equation*}
\cFV [\vrki] - \cFV [\vrkd] \leq  K(\tau) \, ,
\end{equation*}
where $K(\tau)$ is defined in \eqref{K:tau}.
\end{coro}

We are in a position to prove the main inequality. 
\begin{proof}[Proof of Proposition~\ref{prop:main_inequality}]
    Let $\xi_\delta$ be a measure coming from an inexact proximal step and let us denote 
    \[
    \vr_{k+\frac{1}{3}}:=(\prox_{V}^\tau)_\#\vrk,\quad\vr_{k+\frac{2}{3}} := \vr_{k+\frac{1}{3}} * \xi_{\delta},\quad \vr_{k+1} := \vr_{k+\frac{2}{3}} * g_{\tau}\;.
    \]
With this notation we have 
\begin{flalign*}
    \cF[\vrki]-\cF[\nu]=& (\cFV[\vrki]-\cFV[\vr_{k+\frac{2}{3}}]) +(\cFV[\vr_{k+\frac{2}{3}}]-\cFV[\vr_{k+\frac{1}{3}}]) +(\cFV[\vr_{k+\frac{1}{3}}]-\cFV[\nu]) +(\cFE[\vrki]-\cFE[\nu])\\
    =&I_1+I_2+I_3+I_4\,.\end{flalign*}
Now we will bound this four terms separately. 
By Corollary~\ref{coro:dmm5b} we have
\[
 I_1\leq  K(\tau)\,.
\]
By Lemma~\ref{lem:prox_inexact} we have
\[
 I_2 \leq C_V\left(\nu(|\cdot|)+\nu(|\cdot|^{q_V})\delta+\delta\nu(|\cdot|^{q_V-1})+\delta^{q_V}\right)\, \delta\;,
\]
Lemma~\ref{lem:dmm5} {\it (ii)} yields
\[
2\tau I_3 \leq W_2^2 (\vrk, \nu) - W_2^2 (\vr_{k+\frac{1}{3}}, \nu) ,
\]
and Lemma~\ref{lem:dmm5} {\it (iii)} leads to
\[
2\tau I_4  \leq W_2^2 (\vr_{k+\frac{2}{3}}, \nu) - W_2^2 (\hvrki, \nu).
\]
We complete the proof by observing that due to the Lemma~\ref{lem:lem4a}
we have
\[
 W_2^2 (\vr_{k+\frac{2}{3}}, \nu) -W_2^2 (\vr_{k+\frac{1}{3}}, \nu)\leq 2(\nu(|\cdot|)+\vr_{k+\frac{1}{3}}(|\cdot|)+\delta)\delta\,.
\]
    \end{proof}
The last step is to prove error bounds in a strongly convex case.    
\begin{proof}[Proof of Theorem~\ref{theo-main_conv}]
Let us consider the exact proximal Langevin algorithm, i.e. define, for $\hat\vr_0=\vr_0$,  \[
    \hat\vr_{k+\frac{1}{2}}:=(\prox_{V}^\tau)_\#\hvrk,\quad \hvrki := \hat\vr_{k+\frac{1}{2}} * g_{\tau}\;.
    \]
    With the above notation by the same reasoning as in \citet[Proof of Theorem~2]{BCEM} we get
    \begin{equation*}
W_2^2(\hvrk, \minim) \leq \left( 1 - \tfrac{\lambda_V \tau}{2} \right)^k W_2^2 (\vr_0, \minim) + \tfrac{2}{\lambda_V}K(\tau)\,,
\end{equation*}
where $K(\tau)$ is given by \eqref{K:tau}. Given the above inequality we finish the proof by the same lines as in  \citet[Proof of Theorem~6]{BCEM}.
\end{proof}
\section{Explicit formulas for constants}
In this section we present explicit formulas for constants used in the main body of the paper.
\begin{itemize}
    \item The constant $\Ca_m$ from Theorem~\ref{theo:moments} is given by
    \begin{itemize}
        \item for $m\leq 2$ \[\Ca_m= \left[\tfrac{36\kappa^2}{d} + \tfrac{\lambda_V^2}{d}\Ex|X_0|^2 + \tfrac{16\kappa\delta+4 R_V^2\lambda_V + 4\delta R_V\lambda_V  + 8 d}{d} \, \lambda_V\right]^{\frac{m}{2}}\,;\]
        \item for $m>2$, let $a=m-\lfloor m\rfloor$ then
        \begin{multline*}
                    \Ca_m = \Ca_a\Bigg[2^{\lfloor m\rfloor}\ca^{\lfloor m\rfloor}\, \kappa^{\lfloor m\rfloor} d^{-\frac{\lfloor m\rfloor}{2}}+\lfloor m-1\rfloor2\ca(2\kappa\delta +4d)d^{-1}+ d^{-\frac{\lfloor m \rfloor}{2}}\lambda_V^{\lfloor m \rfloor}\left(A_0 +R_V^{\lfloor m \rfloor}\right) \\+\ca d^{-\frac{\lfloor m \rfloor}{2}}\lambda_V^{\lfloor m-1 \rfloor} \left( R_V^{\lfloor m -1\rfloor} +2(2\kappa\delta +4d)R_V^{\lfloor m-2 \rfloor}\lambda_V^{-1} +2^{\lfloor m \rfloor}\kappa\delta^{\lfloor m-1\rfloor}\right) +\ca2^{\lfloor m \rfloor}\widetilde{C}_m
  \lambda_V^{\frac {\lfloor m \rfloor}2} \Bigg]\,;
        \end{multline*}
    \end{itemize}
    \item $C(\nu)$ from Proposition~\ref{prop:main_inequality} can be indicated as follows. Using $|\proxv(x)|\leq|x|$ and by Theorem~\ref{theo:moments} we obtain
    \begin{flalign*}
          C(\nu) &=  C_V\left(\nu(|\cdot|)+\nu(|\cdot|^{q_V})\kappa\tau^{1+\alpha}+\kappa\tau^{1+\alpha}\nu(|\cdot|^{q_V-1})+(\kappa\tau^{1+\alpha})^{q_V}\right) \tau+2(\nu(|\cdot|)+\vr_{k+\frac{1}{3}}(|\cdot|)+\kappa\tau^{1+\alpha})\\
          &\leq C_V\left(\nu(|\cdot|)+\nu(|\cdot|^{q_V})\kappa\tau^{1+\alpha}+\kappa\tau^{1+\alpha}\nu(|\cdot|^{q_V-1})+(\kappa\tau^{1+\alpha})^{q_V}\right) \tau\\
          &\qquad\qquad\qquad\qquad\qquad\qquad+2(\nu(|\cdot|)+\Ca_{1}\min\{1,\lambda_V^{-1}\}{d^{\frac{1}{2}}}+\kappa\tau^{1+\alpha})\,.
    \end{flalign*}
  
  \end{itemize}
\section{Discussion} \label{appendix:simulations}
In this section we present other figures describing observed properties of IPLA in comparison with other approaches.\newline

On Figures~\ref{fig:traceIPLA} and~\ref{fig:trajectories2} we present trace plots of IPLA, TULA, and Metropolis--Hastings Algorithm for Example~1.\newline

\begin{figure}[ht!]
\centering
\includegraphics[width=8cm]{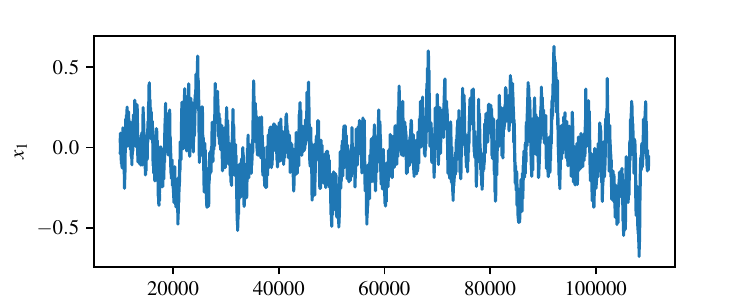}
\caption{Example 1, starting in tail. Trajectory of the 1$^{\mathrm{st}}$ coordinate using IPLA with burn-in time $10\,000$ samples.}
\label{fig:traceIPLA}
\end{figure}

\begin{figure}[ht!]
\centering
\begin{subfigure}{0.48\textwidth}
\includegraphics[width=\textwidth]{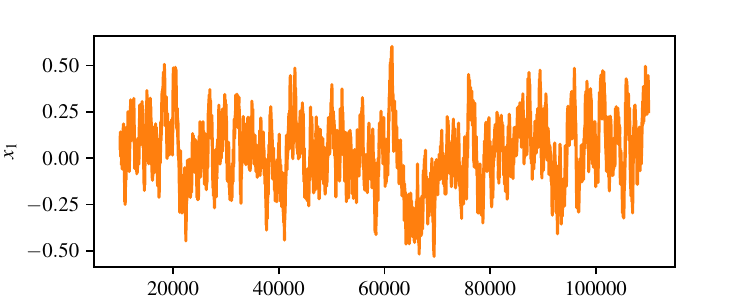}
\caption{TULA}
\end{subfigure}
\begin{subfigure}{0.48\textwidth}
\includegraphics[width=\textwidth]{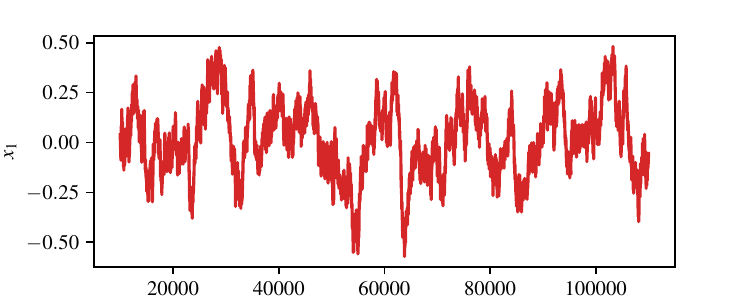}
\caption{MH (reference)}
\end{subfigure}
\caption{Example 1, starting in tail. Trajectory of the 1$^{\mathrm{st}}$ coordinate using TULA (a) and reference Metropolis--Hastings Algorithm (MH) (b) with burn-in time $10\,000$ samples.}
\label{fig:trajectories2}
\end{figure}

Figure~\ref{fig:time_step_minimizer} relates to Figure~\ref{fig:time_step}, but starting in the origin. Figure~\ref{fig:time_step_minimizer}  additionally includes the performance of ULA that quickly blows up.\newline

\begin{figure}
\centering
\begin{subfigure}{0.32\textwidth}
    \includegraphics[width=\textwidth]{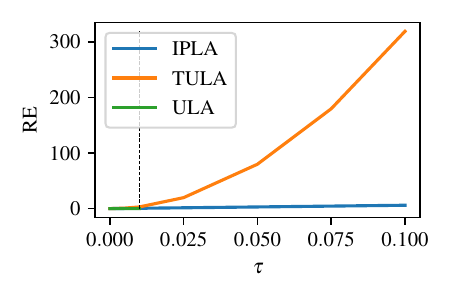}
    \caption{RE}
\end{subfigure}
\begin{subfigure}{0.32\textwidth}
    \includegraphics[width=\textwidth]{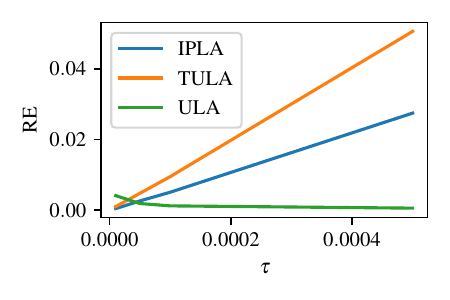}
    \caption{RE (detail)}
\end{subfigure}
\begin{subfigure}{0.32\textwidth}
    \includegraphics[width=\textwidth]{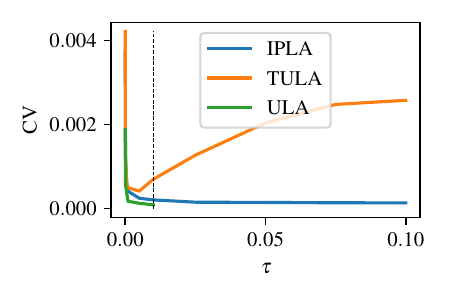}
    \caption{CV}
\end{subfigure}
\caption{Example 1, starting in minimizer. Dependence of RE and CV of IPLA, TULA and ULA on stepsize $\tau$. ULA gives results only for smaller values of $\tau$. Dashed line represents the observed edge of the area of ULA stability ($\tau \approx 0.01$).}
\label{fig:time_step_minimizer}
\end{figure}

Figure~\ref{fig:palaces} illustrates the uncertainty of the Bayesian Image Deconvolution.

\begin{figure}
\begin{subfigure}{0.19\textwidth}
    \includegraphics[width=\textwidth]{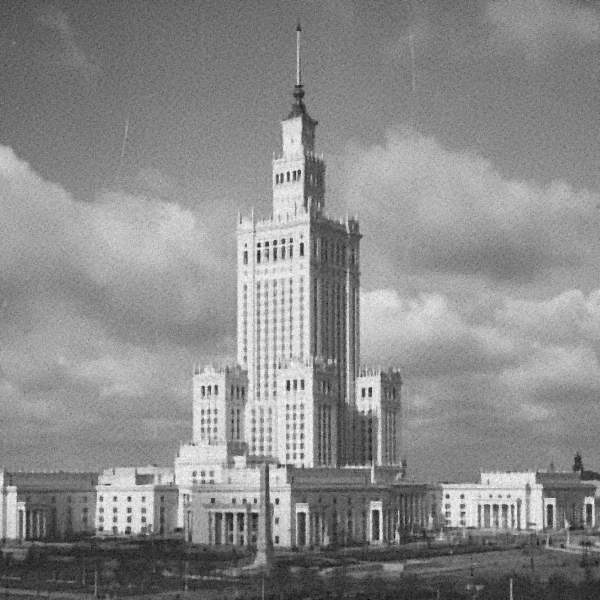}
    \caption{$0.05$}
\end{subfigure}
\begin{subfigure}{0.19\textwidth}
    \includegraphics[width=\textwidth]{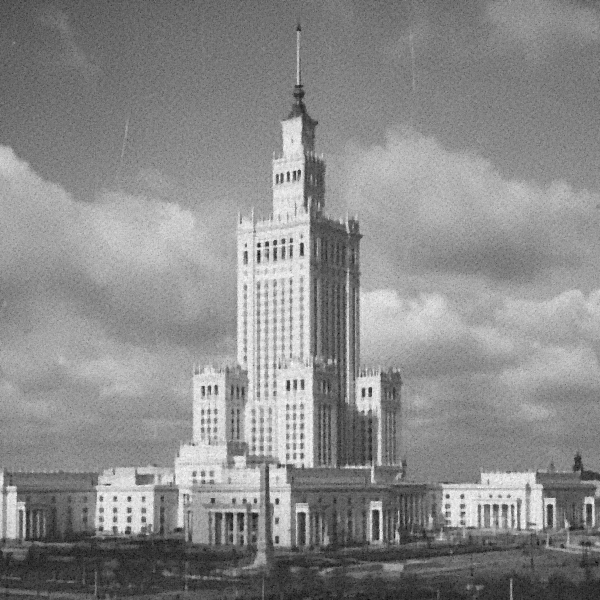}
    \caption{$0.25$}
\end{subfigure}
\begin{subfigure}{0.19\textwidth}
    \includegraphics[width=\textwidth]{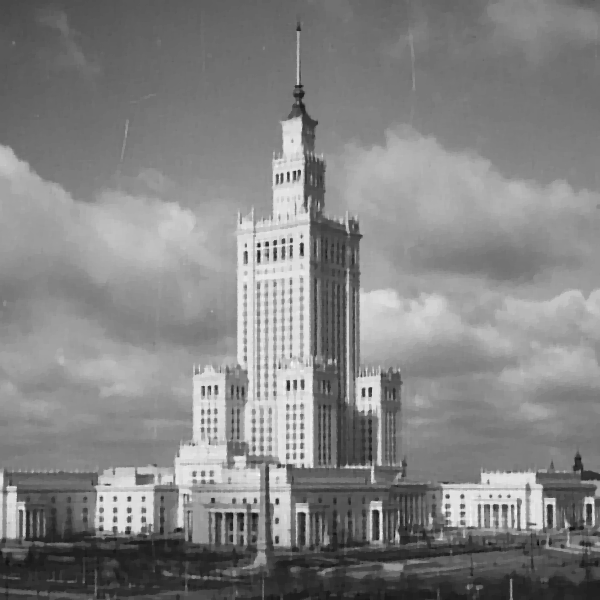}
    \caption{$0.5$}
\end{subfigure}
\begin{subfigure}{0.19\textwidth}
    \includegraphics[width=\textwidth]{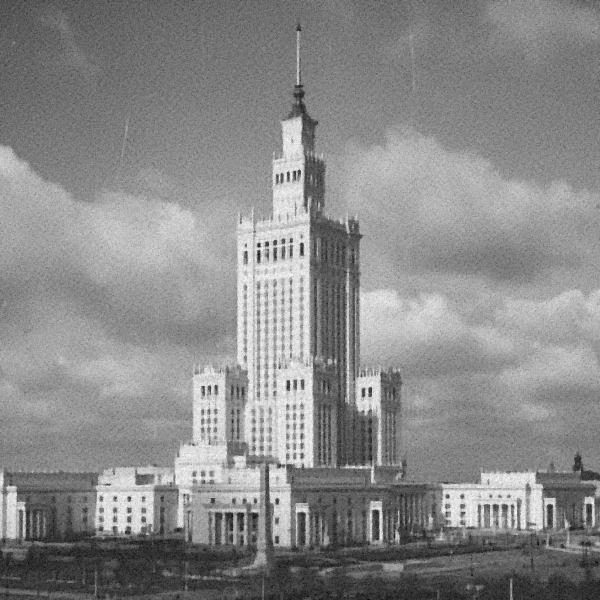}
    \caption{$0.75$}
\end{subfigure}
\begin{subfigure}{0.19\textwidth}
    \includegraphics[width=\textwidth]{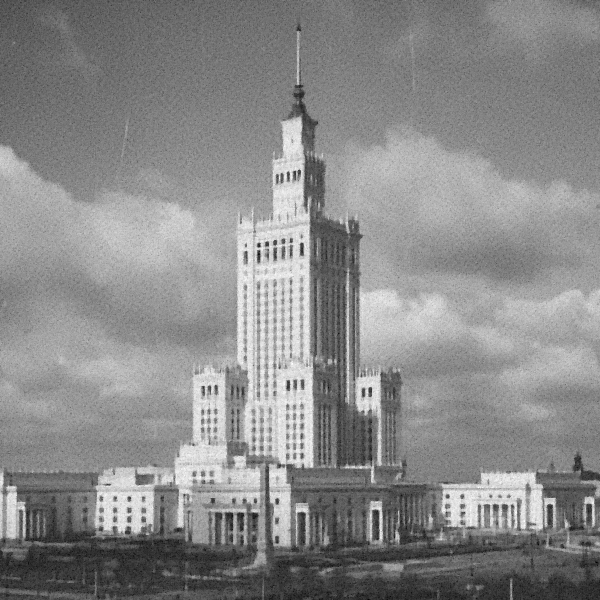}
    \caption{$0.95$}
\end{subfigure}
\caption{Quantiles of the posterior distribution for the Bayesian Image Deconvolution from Example~3 starting in a tail, for $\beta=0.03$ and $\sigma=1$. There are 500 samples drawn with burn in time 50.  }
\label{fig:palaces}
\end{figure}
\section{Proximal operator}
The explicit formula for the operator $\proxv$ is known only in special cases; see, for example, \citet{PS}. However, under our assumptions, an approximation of the proximal operator is a relatively easy task, since the optimized function is strongly convex.  \citet{BCEM} shows that the approximation of the proximal step through the gradient descent algorithm requires the logarithmic number of iterations as a function of precision $\delta$, see \citep[Corollary~5.2, and Remark~5.3]{BCEM}.

\end{document}